\documentclass{article} 
\usepackage{nips12submit_e,times}
\usepackage{amsfonts,amsmath,amssymb,amsthm,epsfig,dsfont}
\usepackage{url,multirow}
\usepackage{bm}
\usepackage{stmaryrd}
\usepackage{algorithmic}
\usepackage{algorithm}
\usepackage{units}
\usepackage{comment}
\usepackage{stackrel}
\usepackage{upgreek}
\usepackage{graphicx}
\usepackage{caption}
\usepackage{subcaption}
\usepackage{mcode}
  \newcommand{\supplementarymaterial}{\par
	  \setcounter{section}{0}%
	  \setcounter{subsection}{0}%
	  \gdef\thesection{S.\arabic{section}}
  }

\def\REAL{{\mathbb{R}}}

\def\bt{\mathbf{t}}
\def\bu{\mathbf{u}}

\def\bx{\mathbf{x}}

\def\vone{{\mathbf{1}}}
\def\vzero{{\mathbf{0}}}

\def\bbeta{{\boldsymbol{\beta}}}
\def\bEta{{\boldsymbol{\eta}}}
\def\bgamma{{\boldsymbol{\gamma}}}

\def\btheta{{\boldsymbol{\theta}}}
\def\bzeta{{\boldsymbol{\zeta}}}
\def\bomega{{\boldsymbol{\omega}}}

\def\bphi{{\boldsymbol{\phi}}}
\def\brho{{\boldsymbol{\rho}}}

\def\bmu{{\boldsymbol{\mu}}}

\newcommand{\EX}{\mathbb{E}}

\newcommand{\POSP}[1]{ { \left[ {#1} \right]_{+} } }

\newcommand{\indicator}{\mathds{1}} 

\DeclareMathOperator*{\argmax}{arg\,max}

\DeclareMathOperator*{\sign}{sign}

\theoremstyle{plain}
\newtheorem{thm}{Theorem}[section]
\newtheorem{lem}[thm]{Lemma}

\newtheorem{cor}[thm]{Corollary}

\makeatletter
\renewcommand*\env@matrix[1][*\c@MaxMatrixCols c]{%
  \hskip -\arraycolsep
  \let\@ifnextchar\new@ifnextchar
  \array{#1}}
\makeatother

\def\const{\text{const}}
\newenvironment{ENUM}{
\begin{enumerate}
  \setlength{\itemsep}{1pt}
  \setlength{\parskip}{0pt}
  \setlength{\parsep}{0pt}
}{\end{enumerate}}

\title{A Bayesian Boosting Model}

\author{
Alexander~Lorbert$^1$, David~M.~Blei$^2$, Robert~E.~Schapire$^2$ and Peter~J.~Ramadge$^1$\\
Departments of Electrical Engineering$^1$ \& Computer Science$^2$\\
Princeton University\\
\texttt{\{alorbert,blei,schapire,ramadge\}@princeton.edu} \\
}

\nipsfinalcopy 

\begin{document}

\maketitle

\begin{abstract} 
We offer a novel view of AdaBoost in a statistical setting.
We propose a Bayesian model for binary classification in which label noise is modeled hierarchically.
Using variational inference to optimize a dynamic evidence lower bound, we derive a new boosting-like algorithm called \textit{VIBoost}.   
We show its close connections to AdaBoost and give experimental results from four datasets.
\end{abstract} 

\section{Introduction} \label{sec:intro}

Boosting, and in particular AdaBoost \cite{schapire2012,freund1996,freund1997}, is an effective method of aggregating classifiers. 
AdaBoost produces a reliable binary classifier and often avoids overfitting.
Nevertheless, it can be sensitive to ``noisy'' data and may severely underperform as a result.
In this paper, we embed binary classification in a Bayesian model and show how it interfaces with the boosting paradigm.
With this model, we can address the vulnerability to noise in a principled way.

Real-world data will almost always include noise, even with binary labels.
In the U.S. Presidential election of 2000, the country was kept in suspense for more than a month while votes were recounted in the state of Florida.
During the recount it had become apparent that the use of the ``butterfly ballot'' had confused voters \cite{toobin2002}.
Probabilistically, we can model a confused voter as one who casts a vote that is \textit{independent} of his/her actual intention.
These votes---borne out of confusion---are considered ``noisy'' and any attempt to learn a voter-to-vote connection, e.g., via boosting, becomes difficult.
However, this does not preclude the extraction of important noise information.
If we can \textit{detect and quantify} the noise properties of a given dataset, then it should be reflected in our expectations of constructing a good classifier.  

In addressing label noise, we have chosen to interpret aggregating classifiers in a fully-Bayesian model.
Once in place, this model lets us incorporate additional latent variables to account for noise.
In our context, noise means that the true label is ignored and randomly reassigned, i.e., it may be inverted.
A learning algorithm such as AdaBoost is sensitive to this type of label perturbation because it focuses on the examples that pose a greater difficulty of classification.
Using this augmented model, we construct an algorithm that performs approximate inference of the posterior distribution associated with the latent variables.
Although the intent is inference, the algorithm is able to produce a binary classifier accompanied by noise statistics that reflect the quality of the learned classifier.
We also show that the algorithm---in its simplest form---reduces to a smoothed version of AdaBoost. 

In developing a Bayesian model for aggregating binary classifiers, we begin with the logistic regression model proposed by \cite{friedman2000}.
Given a set of base classifiers, the latent variables of the model are the weights placed on these base classifiers. 
We then introduce variables to account for label perturbations.
Finally, we use variational inference to estimate the posterior distributions.

Our ideas lead to a new boosting-like algorithm called \textit{VIBoost}---boosting stemming from \underline{v}ariational \underline{i}nference. 
AdaBoost employs a greedy search for incorporating new base classifiers.
Similarly, in VIBoost each main-loop iteration introduces a new base classifier, which induces a new model.
With this new model, variational inference is applied using previous values for a warm start.  
In the process, noise statistics are cultivated.
Our experiments reveal that VIBoost performs on par with AdaBoost and supplies meaningful characterizations of the label perturbations.

Much has been done to cast boosting in a statistical setting.
Friedman~et~al.~\cite{friedman2000} leveraged the logistic regression model and then used a functional gradient to derive a boosting update.
Collins~et~al.~\cite{collins2002} used information geometry to derive AdaBoost and algorithms emerged with exponential and logistic loss objectives. 
Lebanon~\&~Lafferty~\cite{lebanon2002} solidified the relationship between AdaBoost and maximum likelihood via duality.
These ideas led to a Bayesian perspective of boosting and provided a way to incorporate prior knowledge \cite{schapire2002}.

There have been many approaches for handling noise.
For example, Servedio addressed label noise in a PAC learning framework and developed SmoothBoost \cite{servedio2003}. 
Through a statistical formulation, Krause~\&~Singer~\cite{krause2004} addressed noise in the context of symmetric, random label inversions, and devised algorithms to alleviate the resulting adverse effects.
In one of these algorithms they used expectation maximization to construct a classifier while simultaneously updating a noise parameter.
Building on this work, we use variational inference and address label noise in the process.

The paper is organized as follows:
the initial groundwork for the Bayesian model is given in \S\ref{sec:core}.
In \S\ref{sec:logistic} we introduce two probability distributions that will play a role in the model. 
The proposed model is presented in \S\ref{sec:model} and variational inference is applied in \S\ref{sec:vi}.
We discuss the connection to AdaBoost in \S\ref{sec:adaboost}.
We give experimental results in \S\ref{sec:exp} and we conclude in \S\ref{sec:con}.

\section{The Core Model} \label{sec:core}

In the binary classification problem we are given a set of $N$ labeled examples $\{(\bx_n,y_n)\}_{n=1}^N$.
Each example is an element of some space $\mathcal{X}$ and the labels are elements of $\{-1,+1\}$.
In addition to the labeled examples, we also have a set of $M$ base classifiers $\mathcal{F} = \{f_1,\ldots,f_M\}$.
Each element of $\mathcal{F}$ is a function that maps $\mathcal{X}$ to $\{-1,+1\}$.
Additionally, we assume that (a) $h \in \mathcal{F} \Rightarrow -h \notin \mathcal{F}$ and (b) $h_1, h_2 \in \mathcal{F} \Rightarrow \exists \ i \in \{1,\ldots,N\}$ such that $h_1(\bx_i) \neq h_2(\bx_i)$.
These assumptions ensure a finite number of classifiers and prevent identifiability problems.

For a fixed $\bx \in \mathcal{X}$, suppose the logarithm of the ``$+1$''-to-``$-1$'' label odds is given by $F(\bx) = \log \frac{ p(y=+1 \mid \bx) }{ p(y=-1 \mid \bx) }$.
Thus, we can form the conditional probability mass function for the labels as ${ p(y \mid \bx, F) = \frac{ 1 }{ 1 + \exp(-y F(\bx)) } }$.
A label sampled in this way shall be called a \textit{true} label.
Logistic regression models the spatially-variant log-odds-ratio as a weighted sum over all base classifiers, i.e.,  $F(\bx) = \sum_{m=1}^M c_m f_m(\bx)$. 

Consider a model defined by the following generative process:
\begin{samepage}
\begin{ENUM} 
  \item Draw $c_m \stackrel{\text{iid}}{\sim} \mathcal{P}_C$ \ ($m=1,\ldots,M$) \ .
  \item Construct $F = \sum_{m=1}^M c_m f_m$ \ .
  \item Draw $\bx_{1:N}$ independently according to some distribution over $\mathcal{X}$ \ .
  \item Draw $y_n \in \{-1,+1\}$ independently according to ${  p(y_n \mid \bx_n, F) = \frac{1}{1 + \exp[-y_n F(\bx_n) ]} }$ \ . 
\end{ENUM}
\end{samepage}
The graphical model is shown in Figure~\ref{fig:core}.

The latent variables are the base classifier weights $c_{1:M}$.
Using the labeled examples, we seek the posterior distribution over the weights.
A similar approach was posed by Minka \cite{minka2001} with the Bayes Point Machine \cite{herbrich2001}.  
In contrast with our work, the author considered (i) a linear classifier without the notion of base classifiers, (ii) expectation propagation as opposed to variational inference, and (iii) a Gaussian prior for the weights.

The posterior distribution over the weights reflects a compromise of the observed data (\,$\{\bx_n,y_n\}_{n=1}^N$\,) with our prior beliefs (\,$\mathcal{P}_C$\,).
It also has the potential of yielding a classifier via the $M$-dimensional mean or mode, for example.
Combining prior beliefs with observed data is made easier through conjugacy, which is how we propose a distribution for $\mathcal{P}_C$. 
This is the subject of the next section.


\begin{figure}[t!]
\begin{minipage}[b]{0.48\linewidth}
\centering
                \includegraphics[width=0.50\columnwidth]{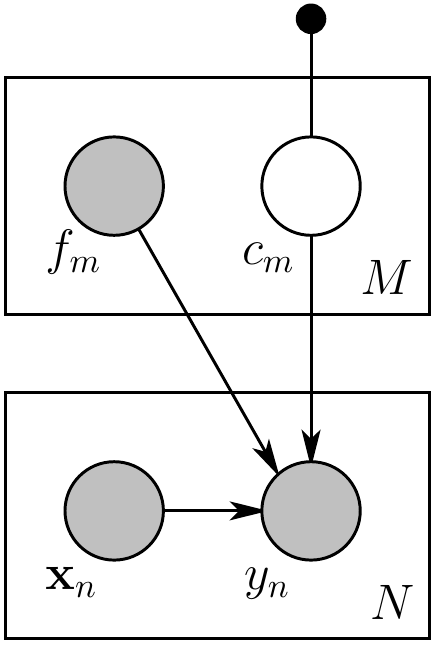}
                \caption{The core graphical model for the boosting problem. Each label depends on the example and log-odds-ratio function. The only latent variables are the base classifier weights.}
                \label{fig:core}
\end{minipage}
\hspace{0.4cm}
\begin{minipage}[b]{0.48\linewidth}
\centering
                \includegraphics[width=1.0\columnwidth]{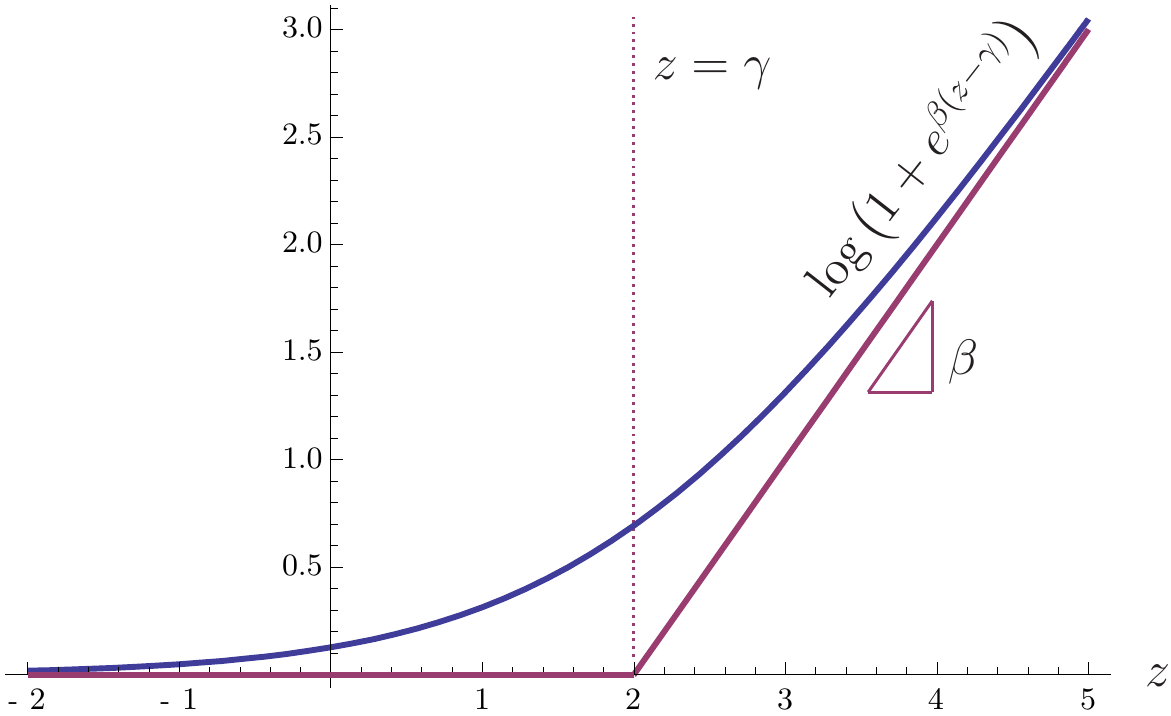}
                \caption{For a given versatile logistic with unit multiplicities, the negative logarithm of one of the product terms is shown above ($\beta\!=\!1$ , $\gamma\!=\!2$). Each curve of this form is tightly lower-bounded by a piecewise-linear function with one knot at $z\!=\!\gamma$. The slopes are $0$ and $\beta$.}
                \label{fig:log1pez}
\end{minipage}
\end{figure}

\section{The Versatile Logistic \& Binary Logistic Distributions} \label{sec:logistic}

We use two conjugate distributions to specify the model described in Figure~\ref{fig:core}.
The first distribution is used as a prior for the weights $c_{1:M}$, and the second is associated with label generation.
For vectors $\bbeta,\bgamma \in \REAL^K$ and $\bmu \in \REAL^K_+$, we define the density over the reals
\begin{align}
  p(z) &\propto \textstyle{\prod_{k=1}^K} \left( \frac{1}{1+\exp[\beta_k (z - \gamma_k)]} \right)^{\mu_k}
  \label{eq:vlog}
\end{align}
to be the \textit{Versatile Logistic Distribution}---written $\text{v-Log}(\bbeta,\bgamma,\bmu)$--- with \textit{slope vector} $\bbeta$, \textit{knot vector} $\bgamma$, and \textit{multiplicity vector} $\bmu$.
Figure~\ref{fig:log1pez} provides the motivation behind this nomenclature.
Define $\bar{\bu} \triangleq \left[ \begin{smallmatrix} +1 \\ -1 \end{smallmatrix} \right] \in \REAL^2$. 
A familiar density is $\text{v-Log}(\bar{\bu},\vzero,\vone)$, which is a logistic distribution.

The density described in \eqref{eq:vlog} is valid if and only if there exists both a positive and negative slope with corresponding positive multiplicity.
Consequently, we must have $K \geq 2$.
Additionally, this distribution is unimodal, so it is reasonable to estimate its mean with an approximate mode.
We prove these facts in \S\ref{supp:sec:vlog}. 
The product represented in \eqref{eq:vlog} relates to a Product of Experts \cite{hinton2002}; however, each factor by itself does not correspond to a valid density.

We now define a probability mass function for the binary random variable $Y$ taking values in $\{-1,+1\}$. 
For scalars $z$, $\beta$, and $\gamma$ we define 
\begin{align}
  p(y) &= \frac{1}{1 + \exp[-y \beta (z - \gamma) ]}
  \label{eq:binarylogistic}
\end{align}
to be the corresponding \textit{Binary Logistic Distribution}, written $\text{b-Log}(z,\beta,\gamma)$.
In comparing \eqref{eq:binarylogistic} to label generation in our model, we see that $\beta$ and $\gamma$ encode base classifier information.   
The versatile logistic and binary logistic are conjugate in the following way:
if $z \sim \text{v-Log}(\bbeta,\bgamma,\bmu)$ and $y_n | z \sim \text{b-Log}(z,\theta_n,\phi_n)$---drawn independently for $n=1,\ldots,N$---then the posterior of $z$ given $y_{1:N}$ is also a versatile logistic with parameters
\begin{align}
  \bbeta' &= [ \beta_1,\ldots,\beta_K,-y_1\theta_1,\ldots,-y_N\theta_N]^T \in \REAL^{K+N} \\
  \bgamma' &= [ \gamma_1,\ldots,\gamma_K,\phi_1,\ldots,\phi_N]^T \in \REAL^{K+N} \\
  \bmu' &= [ \mu_1,\ldots,\mu_K,1,\ldots,1]^T \in \REAL^{K+N} \ .
\end{align}
In the binary classification problem, the b-Log--v-Log conjugacy relationship helps with posterior inference.
By construction, the posterior distribution of the weights is a versatile logistic.



\begin{figure}[t!]
\begin{minipage}[b]{0.48\linewidth}
\centering
                \includegraphics[width=0.95\columnwidth]{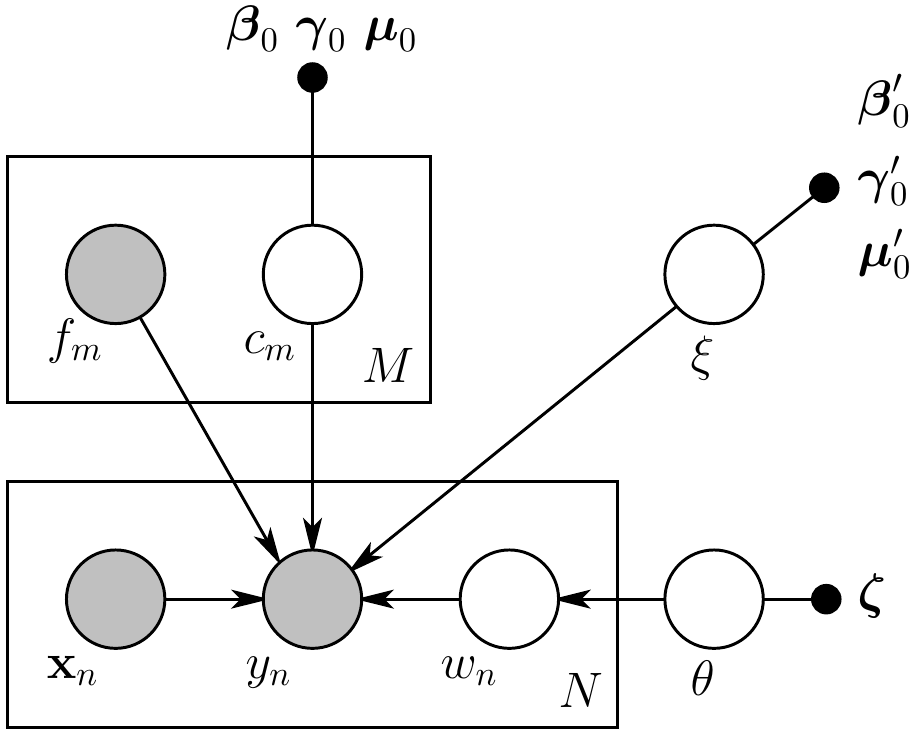}
                \caption{The graphical model for Bayesian boosting.}
                \label{fig:gm}
\end{minipage}
\hspace{0.4cm}
\begin{minipage}[b]{0.48\linewidth}
\centering
                \includegraphics[width=0.95\columnwidth]{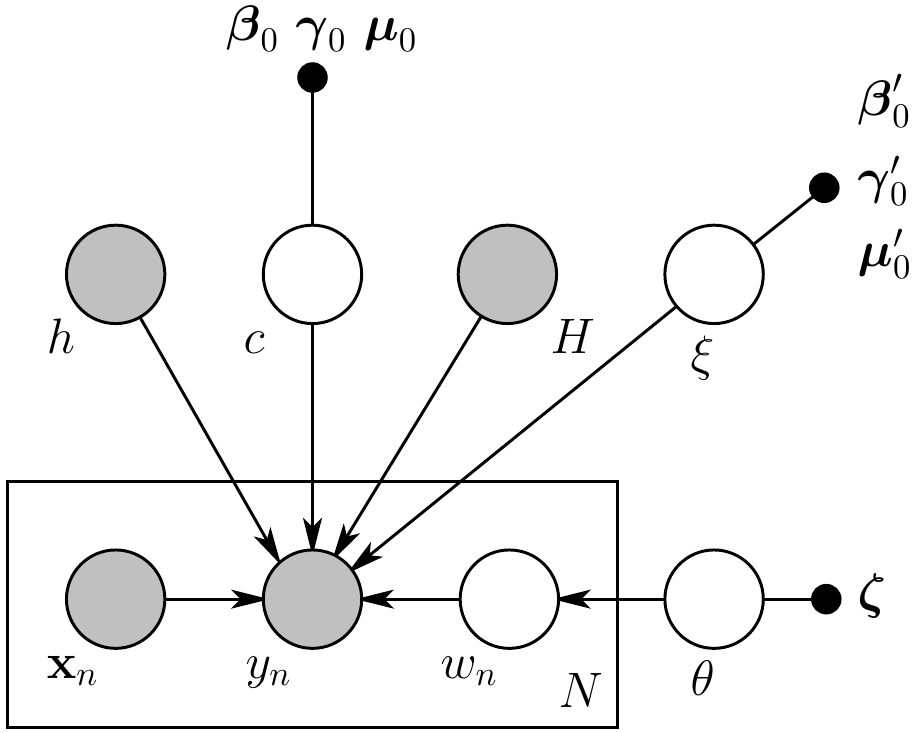}
                \caption{The dynamic graphical model for Bayesian boosting.}
                \label{fig:gmu}
\end{minipage}
\end{figure}

\section{Incorporating Noise} \label{sec:model}

We will now build upon the model presented in \S\ref{sec:core}.
Suppose we fix an instance $\bx \in \mathcal{X}$ and repeatedly generate labels from $p(y \mid \bx , F)$.
If the labels are \textit{true}, then the empirical ratio of plus-to-minus labels will converge to $\exp[F(\bx)]$, i.e., the odds ratio.

On the other hand, according to our model, if the labels are \textit{noisy}, the empirical odds ratio converges to some fixed value, which is independent of $\bx$. 
Let $e^\xi$ be this noise-related odds ratio.
Equivalently, $\xi$ is an instance-independent, static log-odds-ratio, which we refer to as the \textit{noise grade}.
For example, noise grades of $-\infty$, $0$, and $+\infty$, translate to random label assignments of $+1$ with probability $0$, $1/2$, and $1$, respectively.

Let $w$ take on values in $\{0,1\}$ and encode whether a label is true or noisy.
We can merge the two label types into the following conditional label probability:
\begin{align}
  \!\!\!
  p(y \mid w, \bx, \xi, F) &= \frac{1}{1+\exp[-y(w F(\bx) + (1{-}w) \xi)]}
  = \frac{w}{1+\exp[-y F(\bx)]} + \frac{1{-}w}{1 + \exp[-y \xi]} .
  \label{eq:merge}
\end{align}
The role of $w$ selects the label type: true or noisy.
Treating $w$ as a latent variable, we embellish the model of \S\ref{sec:core}:
\begin{samepage}
\begin{ENUM} 
  \item Draw $c_m \stackrel{\text{iid}}{\sim} \text{v-Log}(\bar{\bu},\vzero,\mu_0 \vone)$\ ($m=1,\ldots,M$) \ .
  \item Construct $F = \sum_{m=1}^M c_m f_m$ \ .
  \item Draw $\bx_{1:N}$ independently according to some distribution over $\mathcal{X}$ \ .
  \item Draw $\xi \sim \text{v-Log}(\bar{\bu},\vzero,\mu_0' \vone)$ \ .
  \item Draw $\theta \sim \text{Beta}(\zeta_1,\zeta_2)$ \ \ \ ($\bzeta \in \REAL^2_+$) \ .
  \item Draw $w_n | \theta \stackrel{\text{iid}}{\sim} \text{Bernoulli}(\theta)$ \ ($n=1,\ldots,N$) \ .
  \item Draw $y_n \in \{-1,+1\}$ independently according to
	\begin{align}
	  p(y_n \mid w_n, \bx_n, \xi, F) &=  \frac{1}{1+\exp[-y_n(w_{n}F(\bx_n) + (1{-}w_{n})\xi)]}  \ .
	\end{align}
\end{ENUM}
\end{samepage}
There is now a prior assigned to $c_m$, and the new steps (4-6) model noise. 
The graphical model is depicted in Figure~\ref{fig:gm}.
Although not immediately apparent, this model subsumes label inversion as a form of noise (\S\ref{supp:sec:inversion}).

From the classification standpoint, the primary latent variables of the above generative process are still $c_{1:M}$, or the \textit{weights}.
The latent $w_{1:N}$, or \textit{type selectors}, are responsible for the type of label generated. 
They are drawn independently from $\theta$, the \textit{type prior}.
We can reason that $\theta/(1-\theta)$ represents a signal-to-noise ratio (SNR).
This stems from the expected value of $N\theta$ true labels and $N(1{-}\theta)$ noisy labels (see \S\ref{supp:sec:inversion} for full details).
Alternatively, we can use the prior of $\theta$ for the SNR estimate, yielding $\EX\{\Theta\}/\EX\{1{-}\Theta\} = \zeta_1/\zeta_2$.

\section{Variational Inference} \label{sec:vi}

With a Bayesian model in place, our focus turns to the posterior distribution of the latent variables.
This allows us to construct a classifier by estimating the mean or mode of the posterior weights ($c_m$).
We accomplish this with stagewise variational inference.\footnote{The common approach in this situation is to use a Gibbs sampler. The Gibbs sampler for the graphical model of Figure~\ref{fig:gm} is given in \S\ref{supp:sec:gibbs}. We used Adaptive Rejection Sampling \cite{gilks1992} to sample a v-Log. In practice, this approach was too time consuming, which is why we turned to inference.}

\paragraph{Full Variational Inference.}

Before motivating our stagewise approach, we first review variational inference.
In our model we have the observed variables $\langle \mathcal{F} , \bx_{1:N} , y_{1:N} \rangle$ and the latent variables $\langle c_{1:M} , \xi , w_{1:N} , \theta \rangle$.
We are interested in the posterior $p(c_{1:M} , \xi , w_{1:N} , \theta \,|\, \mathcal{F} , \bx_{1:N} , y_{1:N} )$, which is proportional to the joint $p(c_{1:M} , \xi , w_{1:N} , \theta , \mathcal{F} , \bx_{1:N} , y_{1:N} )$.
In variational inference, we introduce a distribution $q(c_{1:M} , \xi , w_{1:N} , \theta)$ to bound the log of the marginal probability of the observations \cite{wainwright2008},
\begin{align}
  \!\!
   \log  p(  \mathcal{F} , \bx_{1:N} , y_{1:N} \! ) 
	  &\!\geq\! \textstyle{\int} q(c_{1:M} , \xi , w_{1:N} , \theta) \log p(c_{1:M} , \xi , w_{1:N} , \theta , \mathcal{F} , \bx_{1:N} , y_{1:N} \! ) \text{d} c_{1:M} \text{d} \xi \text{d} w_{1:N} \text{d} \theta \notag \\ 
  & \qquad \quad - \textstyle{\int} q(c_{1:M} , \xi , w_{1:N} , \theta) \log q(c_{1:M} , \xi , w_{1:N} , \theta) \text{d} c_{1:M} \text{d} \xi \text{d} w_{1:N} \text{d} \theta \label{eq:elbo}
				  \ .
\end{align}
The right-hand side of \eqref{eq:elbo} is referred to as the \underline{e}vidence \underline{l}ower \underline{bo}und (ELBO).
Using the KL-divergence, we can also write
\begin{align}
\log  p(  \mathcal{F} , \bx_{1:N} , y_{1:N} \! ) = \textsc{kl}(\ q(c_{1:M} , \xi , w_{1:N} , \theta) \parallel \text{posterior}\ ) + \textsc{elbo} \ .
\end{align}
The KL divergence provides a measure of closeness between the auxiliary distribution and the posterior.
We maximize the ELBO with respect to the parameters of $q$, thereby minimizing the KL divergence to the posterior.
We use \textit{mean-field variational inference}, i.e., we assume a factorized $q$: 
\begin{align}
  q(c_{1:M} , \xi , w_{1:N} , \theta) = \textstyle{\prod_{m=1}^M} q(c_m) \cdot q(\xi) \cdot \textstyle{\prod_{n=1}^N} q(w_n) \cdot q(\theta) \ .
  \label{eq:factor}
\end{align}
Each component of the factorized variational distribution has a form and variational parameters.
For example, a reasonable form of $q(c_m)$ is a versatile logistic with variational parameters given by some slope, knot and weight vectors.
Typically, we optimize the parameters with coordinate ascent, updating each in turn, holding the others fixed.  
In our model, this yields the following updates \cite{bishop2006}:
\begin{align}
  \log q^*(c_m) &\leftarrow \EX_q [ \log p( C_{1:m-1}, c_m, C_{m+1:M}, \Xi, W_{1:N}, \Theta, \mathcal{F}, \bx_{1:N}, y_{1:N} ) ] +\const \label{eq:qc} \\
  \log q^*(\xi) &\leftarrow \EX_q [ \log p( C_{1:M}, \xi, W_{1:N}, \Theta, \mathcal{F}, \bx_{1:N}, y_{1:N} ) ] +\const \\
  \log q^*(w_n) &\leftarrow \EX_q [ \log p( C_{1:M}, \Xi, W_{1:n-1}, w_n, W_{n+1:N}, \Theta, \mathcal{F}, \bx_{1:N}, y_{1:N} ) ] +\const \\
  \log q^*(\theta) &\leftarrow \EX_q [ \log p( C_{1:M}, \Xi, W_{1:N}, \theta, \mathcal{F}, \bx_{1:N}, y_{1:N} ) ] +\const  \ .
\end{align}
Each term on the right is a leave-one-out expectation over the latent variables, resulting in a function of the corresponding left-out latent variable.
Running the variational inference algorithm repeatedly cycles through these updates.


This algorithm is not convenient.
The chosen form of the approximate posterior weight distributions is a versatile logistic.
From conjugacy, the number of parameters required to specify each distribution is linear in the number of examples ($N$).
Additionally, we hope to use a large number of base classifiers, even for small datasets.
Thus, for our classification problem, cycling through all auxiliary weight distributions is impractical because integrating over the weights is too much of a computational burden.

\begin{algorithm}[tb]
   \caption{VIBoost}
   \label{alg:viboost}
\begin{algorithmic}
   \STATE {\bfseries Input:} $\{(\bx_n,y_n)\}_{n=1}^N$, $\mathcal{F}$, $\mu_0\in\REAL_+$, $\mu_0'\in\REAL_+$, $\bzeta\in\REAL_+^2$
   \STATE Initialize $H : \mathcal{X} \to \{-1,+1\}$ to the zero function
   \STATE Initialize $\bEta \in \REAL_+^2$, $\bomega \in \REAL_+^2$, and $\bphi \in [0,1]^N$
   \STATE Define $\bbeta(h) \triangleq [+1,-1,-y_1 h(\bx_1),\ldots,-y_N h(\bx_N)]^T$
   \STATE Define $\bgamma(H,h) \triangleq [0,0,-H(\bx_1) h(\bx_1),\ldots,-H(\bx_N) h(\bx_N)]^T$
   \STATE Define $\bmu(\bphi) \triangleq [\mu_0,\mu_0,\phi_{1},\ldots,\phi_{N}]^T$
   \FOR{$t=1$ {\bfseries to} $T$}
   \vskip 0.4ex
   \STATE $h_t \leftarrow \argmax_{h \in \mathcal{F}} \ \vert \text{Mode}[\text{v-Log}(\bbeta(h),\bgamma(H,h),\bmu(\bphi))] \vert$
   \vskip 0.4ex
   \WHILE{\textsc{ELBO} increases significantly}
   \vskip 0.4ex
   \STATE $\alpha_t \leftarrow \text{Mode}[\text{v-Log}(\bbeta(h_t),\bgamma(H,h_t),\bmu(\bphi))]$
   \vskip 0.4ex
   \STATE $\omega_1 \leftarrow \mu_0' + \textstyle{\sum_{n=1}^N} (1-\phi_{n}) \indicator\{y_n = -1\}$ 
   \vskip 0.4ex
   \STATE $\omega_2 \leftarrow \mu_0' + \textstyle{\sum_{n=1}^N} (1-\phi_{n}) \indicator\{y_n = +1\}$ 
   \vskip 0.4ex
   \STATE $\kappa_n \stackrel{1:N}{\leftarrow} \displaystyle{\frac{ \exp\left[ \uppsi(\eta_1) - \uppsi(\eta_2) + \uppsi(\omega_0) - \uppsi(\omega_2)\indicator{\{y_n=+1\}} - \uppsi(\omega_1)\indicator{\{y_n=-1\}} \right] }{ 1 + \exp[-y_n(H(\bx_n) + \alpha_t h_t(\bx_n))]}}$
   \vskip 0.4ex
   \STATE $\phi_{n} \stackrel{1:N}{\leftarrow} \kappa_n / (1 + \kappa_n)$
   \vskip 0.4ex
   \STATE $\eta_1 \leftarrow \zeta_1 + \textstyle{\sum_{n=1}^N} \phi_{n}$
   \vskip 0.4ex
   \STATE $\eta_2 \leftarrow \zeta_2 + \textstyle{\sum_{n=1}^N} (1-\phi_{n})$
   \vskip 0.4ex
   \ENDWHILE
   \STATE $H \leftarrow H + \alpha_t h_t$
   \ENDFOR
   \STATE {\bfseries Output:} classifier\ \ $\sign\{H(\cdot)\}$
\end{algorithmic}
\end{algorithm}

\paragraph{Stagewise Variational Inference.}

To address these issues, we propose a dynamic model over the current static one: with a current estimate of $F = \sum_m c_m f_m$, we introduce a single base classifier and then run variational inference on the latent $\langle c,\xi, w_{1:N}, \theta \rangle$.
The regression counterpart would be Forward Stagewise Regression, a greedy algorithm which finds a sparse subset of covariates and is structurally similar to AdaBoost \cite{friedman2009}.

In each main loop iteration, let $H(\cdot)$ be the current estimate of the true log-odds-ratio $F(\cdot)$ and suppose we have a ``promising'' candidate $h \in \mathcal{F}$ that we wish to merge with $H$.
This promising classifier is found greedily (details below) and, once found, becomes a fixed variable in the model.
Now, rather than $M$ latent weights, we have a single latent weight $c$ corresponding to $h$.
Every update of $H$ induces a new model to which we apply variational inference.
This new, time varying graphical model is featured in Figure~\ref{fig:gmu}.

Let $\mathcal{D}$ denote the evidence, i.e., the observed variables $\bx_{1:N}$, $y_{1:N}$, $H$, and $h$.
At each stage we assume the following distributions:
\begin{align}
  p(c \mid \mathcal{D}) &\approx q( c \mid \bbeta,\bgamma,\bmu ) 
                        \sim \text{v-Log}(\bbeta,\bgamma,\bmu) 
						&
  p(\xi \mid \mathcal{D}) &\approx q( \xi \mid \bomega )
                        \sim \text{v-Log}(\bar{\bu},\vzero,\bomega) \\
  p(w_n \mid \mathcal{D}) &\approx q( w_n \mid \phi_n )
                        \sim \text{Bernoulli}(\phi_n) 
						&
  p(\theta \mid \mathcal{D}) &\approx q( \theta \mid \bEta ) 
                        \sim \text{Beta}(\bEta) \ .
\end{align}

The variational updates and the ELBO are derived in \S\ref{supp:sec:vi} and \S\ref{supp:sec:elbo}, respectively.
The general approach is to isolate the terms of the log-likelihood that feature the variable of interest---all other terms will cancel after normalization and are extraneous.
We then take expectations and attempt to identify a distribution.

The resulting algorithm, \textit{VIBoost}, is presented in Algorithm~\ref{alg:viboost} (\ $\uppsi(\cdot)$ is the digamma function ).
Possible modifications include (a) fixing the number of variational inference iterations so that ELBO calculations are avoided, and (b) setting the $\alpha_t$ once and skipping its update in the variational inference loop.

We greedily select the next base classifier by finding the v-Log posterior with maximal mode.
A large mode suggests that the corresponding weight possesses discriminative classification strength.
We opted for the mode rather than the mean; we now justify this choice.

A v-Log distribution with more than two slope/knot/multiplicity terms has the advantage of being a one-dimensional density, but is cumbersome when evaluating statistics of interest.
Computing the normalization constant, mode and mean require iterative techniques, which can bog down any algorithm.
However, if we replace $\mu_k \log(1 + e^{\beta_k(z-\gamma_k)})$, a summand of the log-density, with the single-tail approximation $\mu_k e^{\beta_k \tau (z - \gamma_k)}$  ($\tau > 0$) we arrive at the modal estimate of
\begin{align}
  \textstyle{
  \alpha = 
  \frac{1}{2\tau\beta} \log \left( \frac{ 
           \sum_{k:\beta_k<0} \mu_k e^{\tau\beta\gamma_k}
     }{ 
           \sum_{k:\beta_k>0} \mu_k e^{-\tau\beta\gamma_k}
		   } \right) } \ ,
		\label{eq:appxmode}
\end{align}
where $\beta = \vert \beta_k \vert$ is constant (\S\ref{supp:sec:cupdate} and \cite{friedman2000}).
For a unimodal distribution, this closed-form expression provides an efficient way of estimating expectations.
Thus, in avoiding numerical integration, Algorithm~\ref{alg:viboost} is performing approximate variational inference.


\section{Relation to AdaBoost} \label{sec:adaboost}

We now compare our algorithm to AdaBoost.
Consider the simpler model of \S\ref{sec:core}, a true-label dataset with prior assignments (Figure~\ref{fig:core}).
This leaves the greedy step of finding the maximal mode in Algorithm~\ref{alg:viboost} and updating $H$ without the variational inference.
We now investigate the approximation supplied by \eqref{eq:appxmode} with $\mu_0 = \tau=1$.
Let $Z=\sum_{n=1}^N e^{-y_n H(\bx_n)}$ and $d_n=e^{-y_n H(\bx_n)} /Z$ so that $\sum_{n=1}^N d_n=1$. 
The approximate mode $\alpha$ is
\begin{align}
  \textstyle{
  \frac{1}{2}
  \log \left( \frac{ 1/Z + (1 - \varepsilon) }{ 1/Z + \varepsilon } \right) }
  \ ,  \label{eq:1oz}
\end{align}
where $\varepsilon =  \sum_{n} d_n \indicator\{h(\bx_n) \neq y_n\} $ is a weighted error ascribed to the new classifier (\S\ref{supp:sec:modeappx}).
When compared to AdaBoost, the update is identical when the $1/Z$ term is not present. 
Effectively, the $1/Z$ term results in a shrinkage of the assigned weights (see Figure~\ref{fig:1Z}).

The variable $Z$ is equal to the current exponential loss.
If $Z$ is small then $1/Z$ is large which leads to a dampened weight assignment (and vice versa).
Assuming the exponential loss decreases with more iterations, the algorithm acts like AdaBoost early on and then becomes more conservative with each iteration.
Quinlan \cite{quinlan1996} incorporated similar smoothing for AdaBoost and described it as ``necessarily ad-hoc''.
In the proposed model, this smoothing results from the prior assignment.
We also note that AdaBoost selects the base classifier that minimizes $\varepsilon$.
From \eqref{eq:1oz}, this coincides with the largest approximation-based mode.

The slopes $\bbeta(h)$ as defined in Algorithm~\ref{alg:viboost} contain individual [mis]matches of the base classifier with the labels, whereas the knots contain individual, weighted [mis]matches of the base classifier with the current log-odds-ratio estimate.
The prior effectively augments the data by inserting two phantom examples.
Each example lies in the zero level set of $H$ as indicated by a knot of $0$ ($H(\bx)h(\bx) = 0 \Rightarrow H(\bx) = 0$).
The slopes of $\pm 1$ presume that the base classifier succeeds in correctly labeling one of the pseudo-examples, while failing with the other. 

Finally, leveraging $d_{1:N}$ we can rewrite Algorithm~\ref{alg:viboost} to use the weighted error $\varepsilon$ rather than a mode search.
Using these errors for ranking the base classifiers---as done in AdaBoost---decreases computation time significantly when searching for a new candidate base classifier.
The greedy search in VIBoost would then closely match AdaBoost's in computation time, thereby leading to an efficient algorithm with a similar runtime to AdaBoost.


\begin{figure}[t]
\begin{minipage}[b]{0.30\linewidth}
\centering
                \includegraphics[width=1.0\columnwidth]{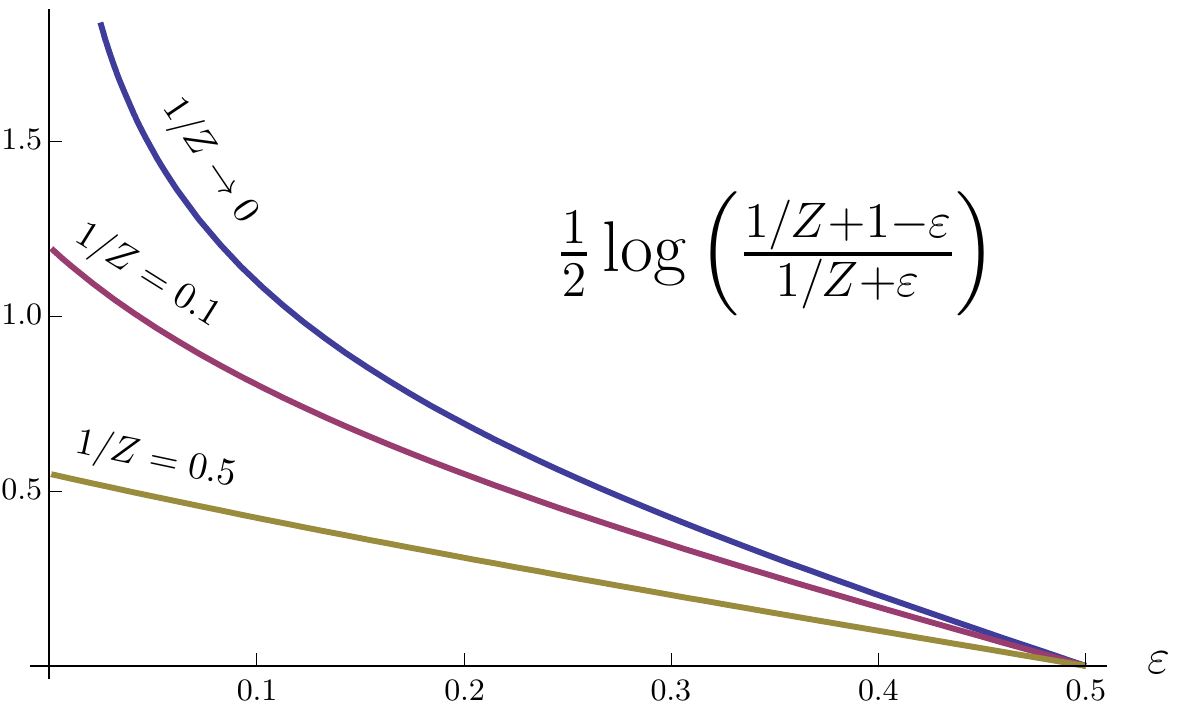}
                \caption{\small The effects of the versatile logistic prior on classifier weights ($\mu_0\!=\!w_{n}\!=\!\tau\!=\!1$)}
                \label{fig:1Z}
\end{minipage}
\hspace{0.03\linewidth}
\begin{minipage}[b]{0.30\linewidth}
\centering
                \includegraphics[width=1.0\columnwidth]{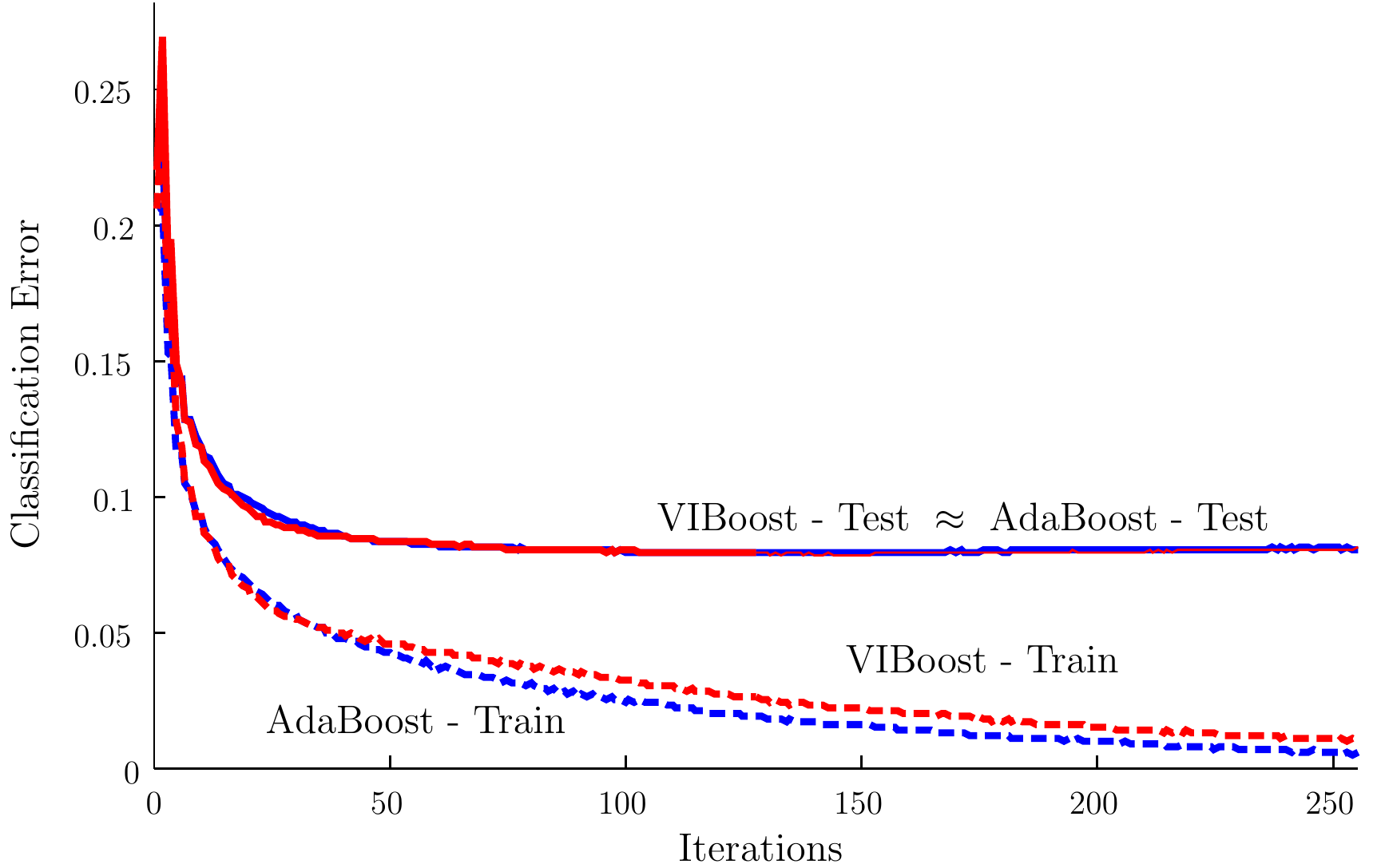}
                \caption{\small Spam dataset, classification error \\ \ }
                \label{fig:spamerr}
\end{minipage}
\hspace{0.03\linewidth}
\begin{minipage}[b]{0.30\linewidth}
\centering
                \includegraphics[width=1.0\columnwidth]{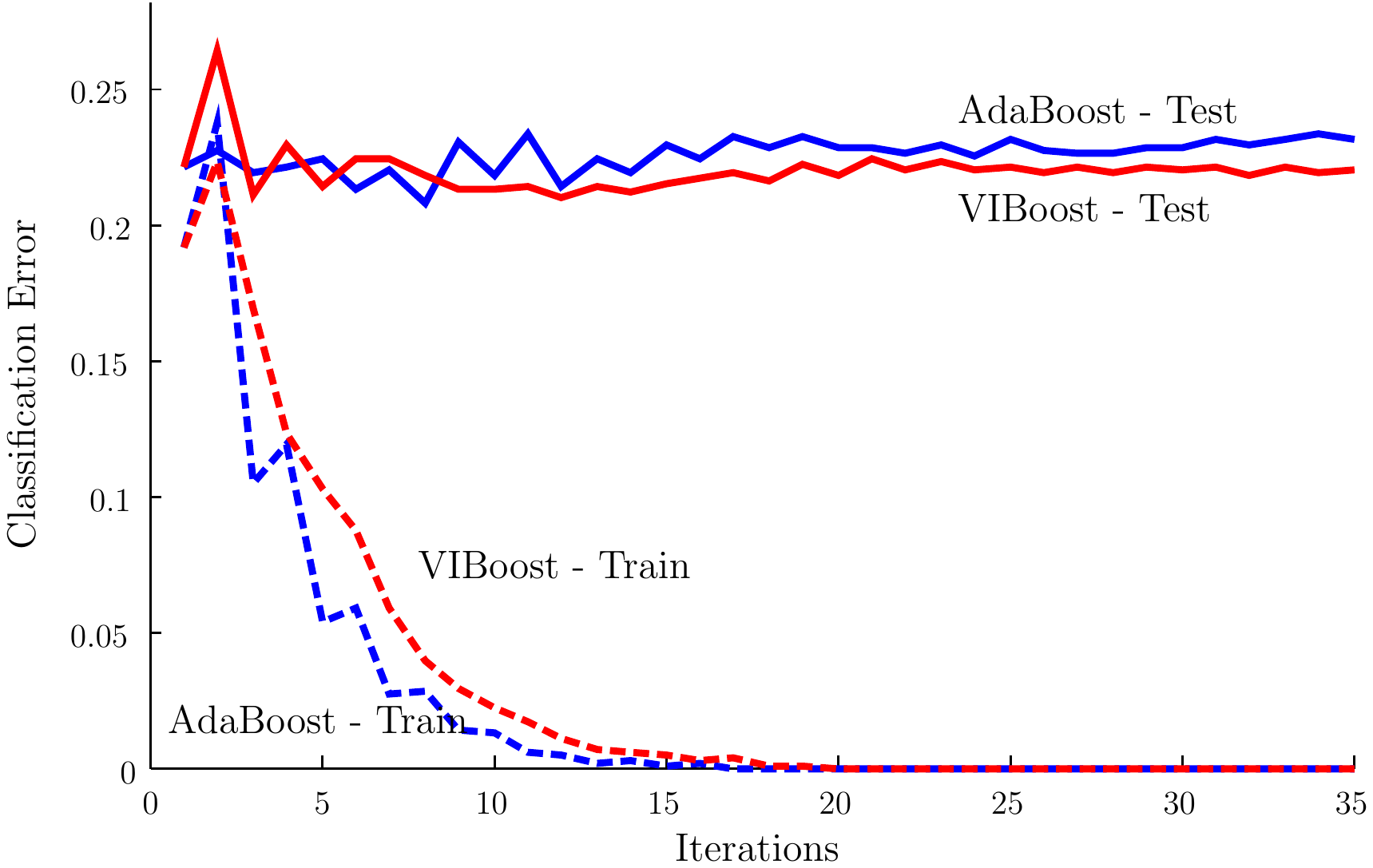}
                \caption{\small State dataset, classification error \\ \ }
                \label{fig:texterr}
\end{minipage}
\end{figure}

\section{Experiments}  \label{sec:exp}

We studied VIBoost on real and synthetic data.
We found that VIBoost works as well as AdaBoost for binary classification.
More importantly, we show that the variables accounting for label noise are a meaningful diagnostic of misfit.
For all experiments, our VIBoost initialization was $\mu_0\!=\!\mu'_0\!=\!\phi_{n}\!=\!\zeta_j\!=\!\eta_j\!=\!\tau\!=\!1$. 
Setting $\tau\!=\!1$ provides the closest means of comparison with AdaBoost. 
For all experiments we used decision stumps as our base classifiers.
Using the variational parameters of the type prior, we use $\eta_1/\eta_2$ for the SNR.
For the noise grade we use $\log(\omega_2/\omega_1)$, the mode associated with the approximate posterior.
All results presented are average values calculated over $40$ runs. 

As VIBoost outputs a classifier, we investigate classifier quality on two real-world, text datasets.
The first dataset is the $57$-feature \textit{spam} dataset \cite{ucimlr}.
With 6,401 examples, each run trained on a random $10\%$ and tested on the remaining $90\%$.  
The second dataset is a \textit{state} dataset \cite{gabrilovich2004} comprising 145 documents with 22,648 features (bag of words).
Instead of the word count, however, we used a present/absent binary value.
Each document relates to Illinois or Michigan.
We trained on $30\%$ and tested on the remaining $70\%$ (random splits).
Error results are featured in Figures \ref{fig:spamerr} and \ref{fig:texterr} and reveal that VIBoost and AdaBoost performed similarly.


In addition to a classifier, VIBoost also provides noise statistics.
Using a synthetic dataset, we now look at the algorithm's estimate of the posterior SNR ($\eta_1/\eta_2$) and the posterior noise grade (\,$\log(\omega_2/\omega_1)$\,) after $50$ iterations.
We simulated 100 examples on the real line with $\mathcal{X} = \{-99,-97,-95,\ldots,+99\}$.
Following the generative process of \S\ref{sec:model}, we constructed the \textit{step} dataset with $F(x)=+\infty$ for $x$ positive and $-\infty$ for $x$ negative ($+1$ label for $x$ positive and $-1$ label for $x$ negative).
With a noise grade of $\log 3 \approx 1.1$, we varied the type prior, $\theta$, of the generative process.
Figures \ref{fig:stepsnr} and \ref{fig:stepxi} respectively show the SNR and noise grade with varying $\theta$.
In a pure-noise situation ($\theta = 0$) the SNR is at its lowest and the noise grade is best estimated.
Conversely, in the absence of noise ($\theta = 1$) the SNR is at its greatest, rendering the noise grade estimate irrelevant. 


The last dataset we considered in this paper was also simulated.
Inspired by \cite{long2010}, we constructed a 1,200-example, 31-feature \textit{Long-Servedio} dataset---a dataset that provably ``breaks'' AdaBoost and many other algorithms with a convex loss minimization.
The details can be found in \cite[\S12.3]{schapire2012} and Matlab code is included in \S\ref{supp:sec:longserv}.
Each run comprised $200$ training examples and $1000$ testing examples (random splits).
Following \cite{schapire2012}, the noise level was set to $0.20$.  
In our context, this translates to a type prior of $0$ (always reassigning a random label) and a noise grade of $\approx -1.4$.
Not surprisingly, AdaBoost and VIBoost do not succeed in finding a decent classifier for this set.
However, the SNR values produced by VIBoost indicate that poor classification should be expected (Figure~\ref{fig:lssnr}).
As a result, the algorithm shifts its focus from classification to noise quantification.

As variational inference navigates through a vast set of auxiliary distributions, our only verifiable means of efficacy is provided by the ELBO (\S\ref{supp:sec:elbo}).
Empirically, we have noticed that ELBO increases are larger in the beginning main-loop iterations.
As the algorithm progresses, changes in the ELBO are quite small and sometimes negative (and small).
The small changes are expected because the composite classifier's accuracy is improving.
We hypothesize that the negative changes stem from the modal approximation used in a $w_n$-update expectation.
Alternatively, the ELBO requires a v-Log normalization constant, which we compute numerically and may be inexact.

\begin{figure}[t]
\begin{minipage}[b]{0.30\linewidth}
\centering
                \includegraphics[width=1.0\columnwidth]{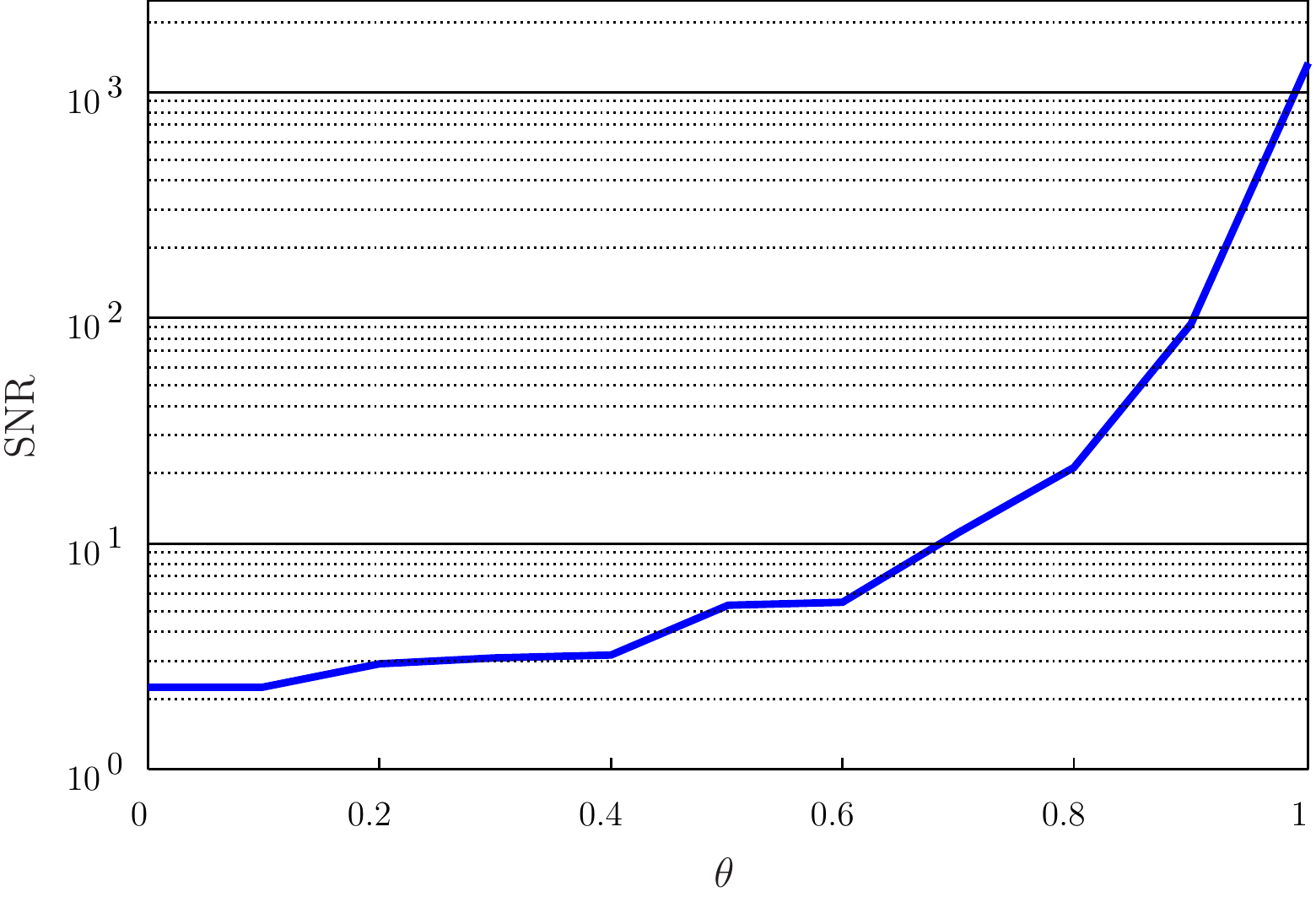}
                \caption{\small Step dataset, Signal-to-noise ratio}
                \label{fig:stepsnr}
\end{minipage}
\hspace{0.03\linewidth}
\begin{minipage}[b]{0.30\linewidth}
\centering
                \includegraphics[width=1.0\columnwidth]{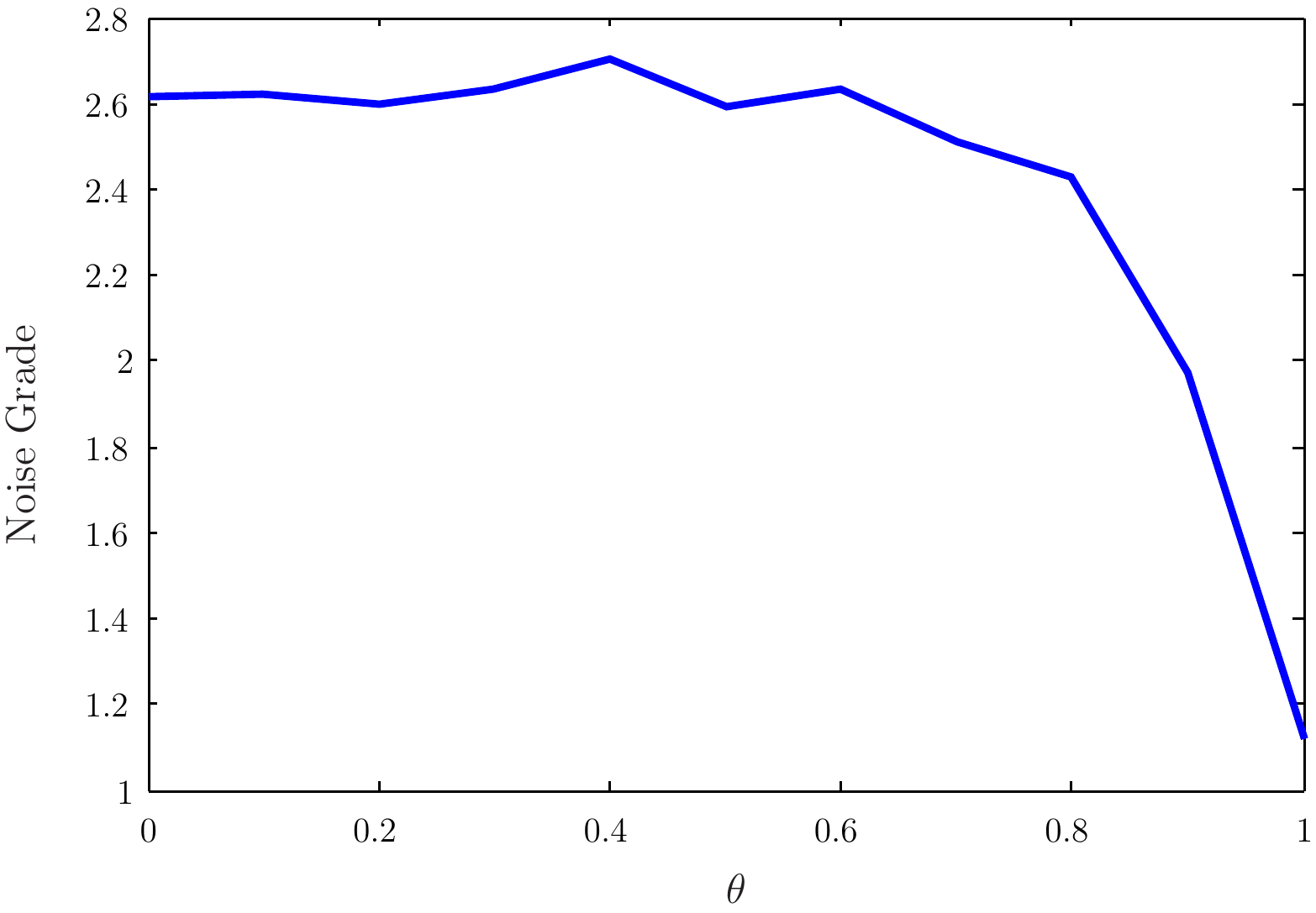}
                \caption{\small Step dataset, Noise Grade}
                \label{fig:stepxi}
\end{minipage}
\hspace{0.03\linewidth}
\begin{minipage}[b]{0.30\linewidth}
\centering
                \includegraphics[width=1.0\columnwidth]{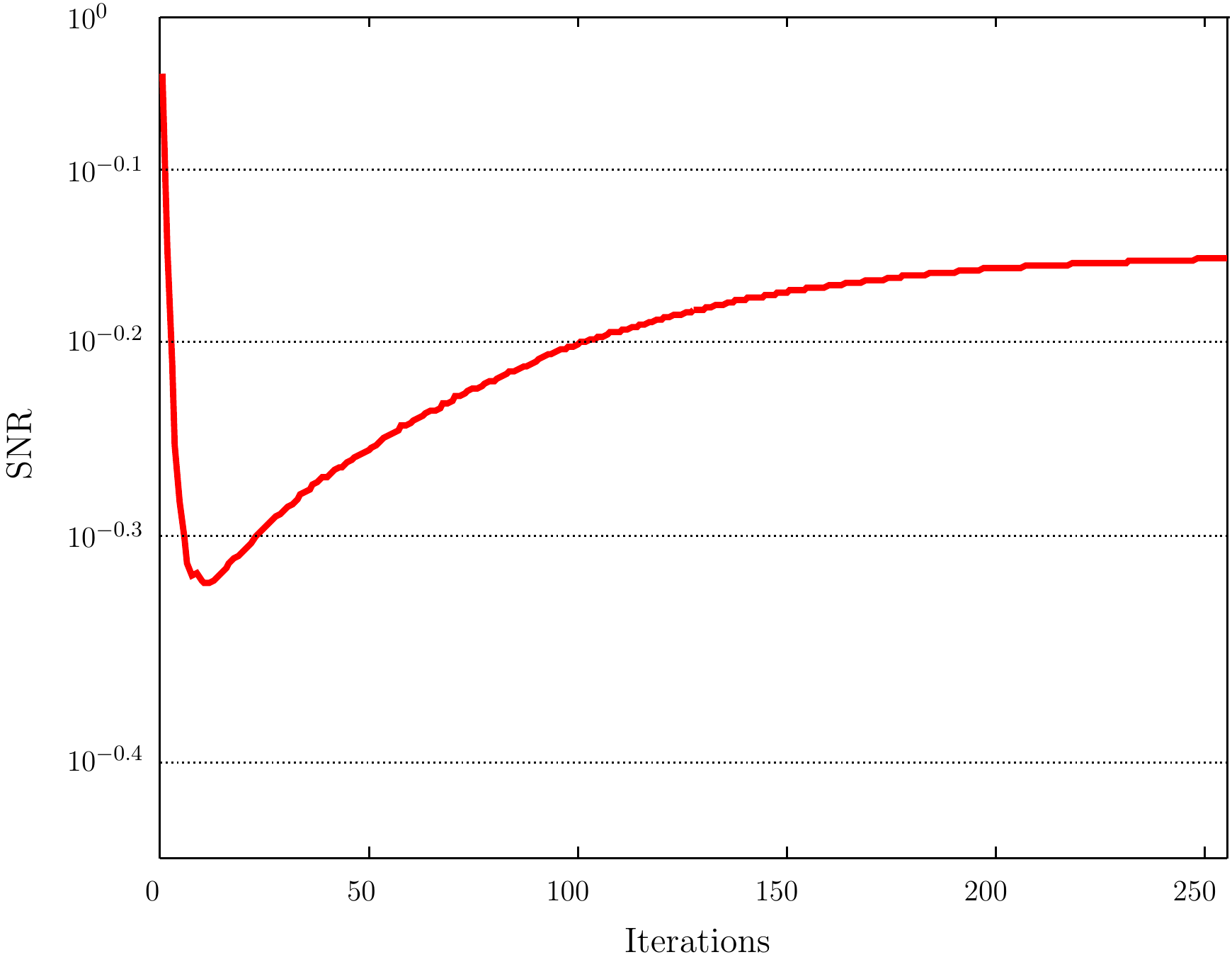}
                \caption{\small Long-Servedio dataset, Signal-to-noise ratio}
                \label{fig:lssnr}
\end{minipage}
\end{figure}

%
\section{Conclusion}  \label{sec:con}

We have developed a new boosting-like algorithm.
VIBoost attempts to fit a posterior distribution by applying variational inference to a dynamic model.
We began with a model centered around the binary classification problem and augmented it hierarchically to account for noise.

We did not set out to improve AdaBoost.
In addition to forming a binary classifier, the Bayesian model facilitated a label noise extension and we were able to extract information beyond classification.
We have observed experimentally that a good classifier is accompanied by a large SNR.
The SNR may explain why a poor classifier is returned by the learning algorithm.
We demonstrated this by analyzing the Long-Servedio dataset.

This model and accompanying algorithm are fertile ground for future work.
This paper did not address multi-class problems or regression.
We can also extend our model by forming connections between instances, base classifiers, and classifier weights (currently, these three features are conditionally independent given the labels).
We can also form dependencies between instances and label types, modeling varying levels of noise throughout the instance space.

\nocite{jordan1995}
\nocite{chen1995}
{ \small
\bibliography{refs}
\bibliographystyle{plain}
}


\newpage
\supplementarymaterial


\section{Properties of the Versatile Logistic}   \label{supp:sec:vlog}

The set of nonnegative reals is denoted $\REAL_+$.
For vectors $\bbeta,\bgamma\in \REAL^K$ and $\bmu \in \REAL_+^K$ let
\begin{align}
  f(z) &= \prod_{k=1}^K \left( \frac{1}{1+\exp[\beta_k(z - \gamma_k)]} \right)^{\mu_k} = \prod_{k=1}^K r_k(z) \ ,
  \label{supp:eq:vlogx1}
\end{align}
where $r_k(z) = \left( \frac{1}{1+\exp[\beta_k(z - \gamma_k)]} \right)^{\mu_k}$.
We note that 
\begin{align}
  0 &\leq \frac{1}{1+\exp[\beta_k(z - \gamma_k)]} \leq 1 \\ 
  0 &\leq \left(\frac{1}{1+\exp[\beta_k(z - \gamma_k)]}\right)^{\mu_k} = r_k(z) \leq 1 
\end{align}
and
\begin{align}
  r_k(z) &= \left(\frac{1}{1+\exp[\beta_k(z - \gamma_k)]}\right)^{\mu_k} 
      \leq \left(\frac{1}{\exp[\beta_k(z - \gamma_k)]}\right)^{\mu_k} 
	  = \exp[-\mu_k\beta_k(z - \gamma_k)] \ . \label{supp:eq:vlogx2}
\end{align}

\begin{lem}
  The integral $I=\int_{-\infty}^{+\infty} f(z) dz$ is finite if and only if there exists an $i$ and $j$ such that $\beta_i > 0$, $\mu_i>0$, $\beta_j < 0$, and $\mu_j > 0$.
\end{lem}
\begin{proof}
  If there is a $k$ such that $\beta_k = 0$ or $\mu_k=0$ then $r_k(z)$ is a constant and does not contribute to the finiteness of the integral. 
  Therefore, without loss of generality we can assume that none of the $\beta_k$ or $\mu_k$ are zero.
  \\ ($\Rightarrow$) With $r_k(z)$ nonnegative and bounded by $1$ we have $f(z) \leq r_k(z)$ for all $k$ and $z$.   
  It follows that
  \begin{align}
	I &\leq \int_{-\infty}^{+\infty} \min_k r_k(z) dz \\
	&\leq \int_{-\infty}^{+\infty} \min\{ r_i(z) , r_j(z) \} dz \\
	&= \int_{-\infty}^{t} r_j(z) dz +  \int_{t}^{+\infty} r_i(z) dz \ ,
  \end{align}
  where $t$ is the unique solution to $r_i(t) = r_j(t)$. 
  To show that $t$ exists, let $g(z) = r_j(z) - r_i(z)$.
  We have $\lim_{z \to -\infty} g(z) = 0 - 1 = -1$ and $\lim_{z \to +\infty} g(z) = 1 - 0 = 1$.
  Also, 
  \begin{align}
	g'(z) &= -\mu_j \beta_j (1+\exp[\beta_j(z-\gamma_j)])^{-\mu_j-1}(\exp[\beta_j(z-\gamma_j)]) \notag \\ &\qquad + \mu_i \beta_i (1+\exp[\beta_i(z-\gamma_i)])^{-\mu_i-1}(\exp[\beta_i(z-\gamma_j)]) > 0 \ .
  \end{align}
  With $g$ increasing it crosses the $z$-axis once (intermediate value theorem) at $t$.
  From \eqref{supp:eq:vlogx2} we have
  \begin{align}
	I &\leq \int_{-\infty}^{t} \exp[-\mu_j\beta_j(z-\gamma_j)] dz 
	       +  \int_{t}^{+\infty}  \exp[-\mu_i\beta_i(z-\gamma_i)] dz
  \end{align}
  is finite because each integral on the right-hand side is an integral of an exponential tail.
  \\ ($\Leftarrow$) Assume all of the $\beta_k$ are negative.
  As $z \to +\infty$ $f(z)$ will approach $1$.
  So there exists a $z_1 \in \REAL$ such that $f(z) \geq 1/2$ when $z \geq z_1$, leading to a divergent integral.
  The analogous case can be made for the $\beta_k$ all positive.
  The only other choice is that there is a $\beta_i>0$ and $\beta_j<0$.
\end{proof}

\begin{cor}
  The density $p(z) \propto f(z)$ is valid if and only if there exists an $i$ and $j$ such that $\beta_i > 0$, $\mu_i>0$, $\beta_j < 0$, and $\mu_j > 0$.
\end{cor}

\begin{lem}
  $\log r_k(z)$ is concave.
\end{lem}
\begin{proof}
  If $\beta_k = 0$ or $\mu_k=0$ then $r_k(z)$ is a constant function, which is concave.
  Otherwise, we have
  \begin{align}
	\frac{d}{dz} \log r_k(z) &= -\mu_k \frac{ \beta_k\exp[\beta_k(z-\gamma_k)] }{1 + \exp[\beta_k(z-\gamma_k)] } = -\mu_k \frac{ \beta_k }{1 + \exp[-\beta_k(z-\gamma_k)] } \\
	\frac{d^2}{dz^2} \log r_k(z) &= -\mu_k \frac{ \beta_k^2 \exp[-\beta(z-\gamma_k)] }{(1 + \exp[-\beta_k(z-\gamma_k)])^2 } < 0
  \end{align}
  proving concavity.
\end{proof}

\begin{lem}
  The distribution $\text{v-Log}(\bbeta,\bgamma,\bmu)$ is unimodal. 
\end{lem}
\begin{proof}
  If $p(z)$ is the associated density then we wish to show that $p(z)$ has one critical point.
  Since $p(z)>0$, $\log p(z)$ will have the same critical points as $p(z)$ because $d(\log p(z))/dz = p'(z)/p(z)$.
  From the previous Lemma, $\log p(z)$ is concave.
  Being a valid density over the reals, concavity ensures that $\log p(z)$ will have one critical point as it increases and then decreases.
\end{proof}

\subsection{On $\text{v-Log}(\beta \bar{\bu} , \gamma \vone , [\mu_1,\mu_2]^T)$} \label{supp:sec:vlog11}

Recall that $\bar{\bu} \triangleq [+1,-1]^T$ and we will assume $\beta > 0$.
Let $V \sim \text{Beta}(\mu_1,\mu_2)$ and $g(v) = \gamma + \frac{1}{\beta}\log\left(\frac{1}{v}-1\right)$.
The function $g$ is monotonic and maps $[0,1]$ to $\REAL$.
The inverse function is $g^{-1}(z) = \frac{1}{1 + \exp[\beta(z-\gamma)]}$.
Note that $1-g^{-1}(z) =  \frac{1}{1 + \exp[-\beta(z-\gamma)]}$.
If $Z = g(V)$ then we have
\begin{align}
  p_Z(z) &= \frac{ p_V(v) }{ \vert g'(v) \vert } 
  = \frac{\Upgamma(\mu_1+\mu_2)}{\Upgamma(\mu_1)\Upgamma(\mu_2)} \frac{ v^{\mu_1-1}(1-v)^{\mu_2-1} }{ \frac{1}{\beta v (1-v)} }
  = \frac{\Upgamma(\mu_1+\mu_2)}{\Upgamma(\mu_1)\Upgamma(\mu_2)} \beta v^{\mu_1}(1-v)^{\mu_2} 
  \\ &
  = \frac{\Upgamma(\mu_1+\mu_2)}{\Upgamma(\mu_1)\Upgamma(\mu_2)}  \beta \left( \frac{1}{1 + \exp[\beta(z-\gamma)]} \right)^{\mu_1} \left( \frac{1}{1 + \exp[-\beta(z-\gamma)]} \right)^{\mu_2} \ . 
\end{align}
Thus, $Z \sim \text{v-Log}(\beta \bar{\bu} , \gamma \vone , [\mu_1,\mu_2]^T)$.
The normalization constant is 
\begin{align}
  \frac{1}{\beta}\frac{\Upgamma(\mu_1)\Upgamma(\mu_2)}{\Upgamma(\mu_1+\mu_2)}  \ .
\end{align}

\subsection{The Exponential Family}

A density of the form 
\begin{align}
  p(z \mid \bEta) = h(z) \exp\{ \bEta^T \bt(z) - a(\bEta) \}
  \label{supp:eq:conj}
\end{align}
is said to belong to the exponential family. 
If we set $h(z) = 1$, define the sufficient statistics
\begin{align}
  \bt(z) \triangleq \begin{bmatrix} \log\left(1 + e^{\beta_1 (z - \gamma_1)}\right) \\ 
	                                                 \vdots \\
								   \log\left(1 + e^{\beta_K (z - \gamma_K)}\right) \end{bmatrix}
								   \in \REAL^{K} \ \ ,
  \label{supp:eq:suffstat}
\end{align}
and set the natural parameters $\bEta = -\bmu \in \REAL^K$, then $\text{v-Log}(\bbeta,\bgamma,\bmu)$ is a member of the exponential family.
Observe:
\begin{align}
  h(z) \exp\{ \bEta^T \bt(z) - a(\bEta) \} 
    &= 
	\exp\left\{ -\sum_{k=1}^K \mu_k \log\left(1 + e^{\beta_k (z - \gamma_k)}\right) - a(-\bmu) \right\}
	\\ &=
	e^{-{a(-\bmu)}} \prod_{k=1}^K \left( \frac{1}{1 + e^{\beta_k (z - \gamma_k)}} \right)
	\\ &\propto
	\prod_{k=1}^K \left( \frac{1}{1 + e^{\beta_k (z - \gamma_k)}} \right) \ \ .
\end{align}

\newpage
\section{Accounting for label inversions}   \label{supp:sec:inversion}

Let $\brho = [\rho_1,\rho_2,\rho_3]^T$ denote a probability vector and let $v \in \{-1,+1\}$ denote a label.
We now form $y$, a stochastic mapping of $v$, as follows: 
\begin{enumerate}
  \item With probability $\rho_1$, $y \leftarrow v$ [equality]
  \item With probability $\rho_2$, $y \leftarrow -v$ [inversion]
  \item With probability $\rho_3$, $y \leftarrow \begin{cases} +1 & \text{\ \ w/ prob\ } r \\ -1 & \text{\ \ w/ prob\ } 1-r \end{cases}$\ \  [independent Bernoulli trial]
\end{enumerate}

As outlined above, the formation of $y$ from $v$ possesses $3$ degrees of freedom: $r$ and two elements of $\brho$. 
We can create the equivalent stochastic mapping: 
\begin{enumerate}
  \item With probability $\theta=2\rho_1 + \rho_3 - 1$, $y \leftarrow v$ [equality]
  \item With probability $\bar{\theta} = 1-\theta = 2-2\rho_1-\rho_3$, \\[2ex] $\quad$ $y \leftarrow \begin{cases} +1 & \text{\ \ w/ prob\ } s = \frac{1-\rho_1-\rho_3(1-r)}{2-2\rho_1-\rho_3} \\ -1 & \text{\ \ w/ prob\ } 1-s \end{cases}$\ \  [independent Bernoulli trial] \ \ .
\end{enumerate}

To show equivalence, we have:
\begin{align}
  P( Y = +1 \mid V = +1 ) &= \theta + \bar{\theta} s \quad [= 1-a]
    \\ &= 
	(2\rho_1 + \rho_3 - 1) + (2-2\rho_1-\rho_3)\frac{1-\rho_1-\rho_3(1-r)}{2-2\rho_1-\rho_3} 
	\\ &= \rho_1 + \rho_3 r \\
	P( Y = -1 \mid V = +1 ) &= \bar{\theta} (1-s)  \quad [= a] 
	\\ &=
	(2-2\rho_1-\rho_3)\left( 1 - \frac{1-\rho_1-\rho_3(1-r)}{2-2\rho_1-\rho_3} \right)
	\\ &= 1 - \rho_1 - \rho_3 + \rho_3 (1-r) 
	\\ &= \rho_2 + \rho_3 (1-r)  \\
	P( Y = +1 \mid V = -1 ) &= \bar{\theta} s  \quad [= b]
    \\ &=
	(2-2\rho_1-\rho_3)\frac{1-\rho_1-\rho_3(1-r)}{2-2\rho_1-\rho_3} 
	\\ &= 1 - \rho_1 - \rho_3 + \rho_3 r 
	\\ &= \rho_2 + \rho_3 r \\
	P( Y = -1 \mid V = -1 ) &= \theta + \bar{\theta} (1-s)  \quad [= 1-b]
    \\ &= 
	(2\rho_1 + \rho_3 - 1) + (2-2\rho_1-\rho_3)\left(1 -\frac{1-\rho_1-\rho_3(1-r)}{2-2\rho_1-\rho_3} \right)
	\\ &= (2\rho_1 + \rho_3 - 1) + \rho_2 + \rho_3 (1-r) 
	\\ &= \rho_1 + \rho_3 (1-r) \ \ .
\end{align}
The above also describes a Binary Asymmetric Channel (BAC) \cite{moser2009} with parameters $a$ and $b$ (See Figure~\ref{fig:bac}).
When we expect a balanced dataset, i.e., the expected number of $+1$ and $-1$ labels are equal, the ratio of true labels to noisy labels is
\begin{align}
  \frac{ \frac{N}{2}(1-a) + \frac{N}{2}(1-b)}{\frac{N}{2}a+ \frac{N}{2}b}
    &=
	\frac{2 - (a+b)}{(a+b)}
	=
	\frac{2 - \bar{\theta}}{\bar{\theta}}
	=
	\frac{1 + \theta}{1 - \theta} \ ,
\end{align}
which is lower bounded by $1$.
Looking ahead, $\theta$ represents a random quantity with expectation $\frac{\eta_1}{\eta_1+\eta_2}$.
Using this expectation in place of $\theta$, the above ratio becomes
\begin{align}
  \frac{1 + \frac{\eta_1}{\eta_1+\eta_2}}{1 - \frac{\eta_1}{\eta_1+\eta_2}}
	  &=
	\frac{2 \eta_1 + \eta_2}{\eta_2}
	 =
	1 + 2\frac{\eta_1}{\eta_2} 
	 =
	 1 + 2\frac{ \EX[ \Theta ]}{ \EX[  1- \Theta ]} ,
\end{align}
thus motivating the use of $\frac{ \EX \Theta }{ \EX[  1- \Theta ]}$.
\begin{figure}[ht]
\begin{center}
\centerline{\includegraphics[scale=1.0]{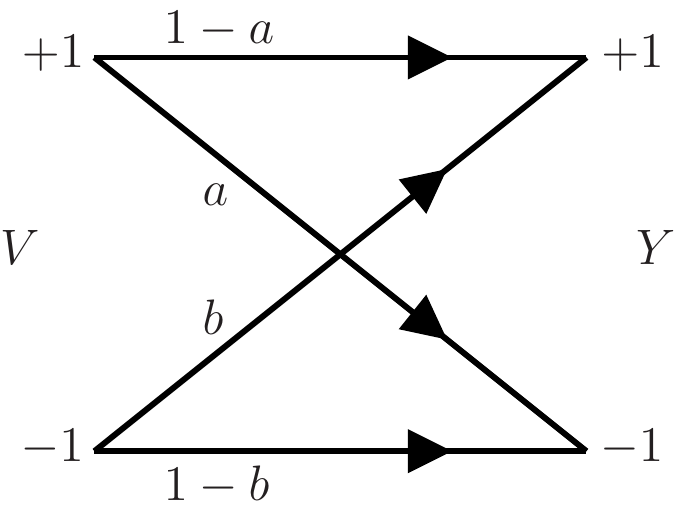}}
\caption{The Binary Asymmetric Channel, $P( Y \mid V)$.  }
\label{fig:bac}
\end{center}
\end{figure} 

\newpage
\section{The Gibbs Sampler}  \label{supp:sec:gibbs}

This section refers to the original noise model.
The joint, denoted $\mathcal{J}$, is
\begin{align}
  \mathcal{J} &\propto 
	\left(\frac{1}{1+e^\xi}\right)^{\mu_0'} \left(\frac{1}{1+e^{-\xi}}\right)^{\mu_0'}
	 \times
	 \theta^{\zeta_1 - 1} (1-\theta)^{\zeta_2 - 1}
	  \notag \\ &\quad
	 \times
	 \prod_{m=1}^M \left(\frac{1}{1+e^{c_m}}\right)^{\mu_0} \left(\frac{1}{1+e^{-{c_m}}}\right)^{\mu_0}
	 \times
	 \prod_{n=1}^N \theta^{w_{n}} (1-\theta)^{1-w_{n}}
	  \\ &\quad
	 \times
	 \prod_{n=1}^N 
	   \left(\frac{1}{1 + \exp\left[-y_n \sum_{m=1}^M c_m f_m(\bx_n)\right]}  \right)^{w_{n}}
	   \left(\frac{1}{1 + \exp\left[-y_n \xi \right]}  \right)^{1-w_{n}} \notag
  \label{supp:eq:mcmcjoint}
\end{align}

For variable $z$, let $\mathcal{J}[z]$ denote the distribution $z$ with all other variables fixed.
Starting with $c_i$, we have
\begin{align}
  \mathcal{J}[c_i] &\propto 
	    \left(\frac{1}{1+e^{c_i}}\right)^{\mu_0} \left(\frac{1}{1+e^{-{c_i}}}\right)^{\mu_0}
	  \\ &\quad
	 \times
	 \prod_{n=1}^N 
	   \left(\frac{1}{1 + \exp\left[-y_n \left( c_i f_i(\bx_n) + \sum_{m \neq i} c_m f_m(\bx_n)\right)\right]}  \right)^{w_{n}}
	  \\ &\quad
	 \\  &\equiv
	 \text{v-Log}\left(
	   \begin{bmatrix} +1 \\ -1 \\ -y_1 f_i(\bx_1) \\ \vdots \\ -y_n f_i(\bx_N)  \end{bmatrix} , 
		 \begin{bmatrix}  0 \\  0 \\ -\tilde{f}_i(\bx_1) f_i(\bx_1) \\ \vdots \\ -\tilde{f}_i(\bx_N) f_i(\bx_N)  \end{bmatrix} , 
		 \begin{bmatrix}  \mu_0 \\  \mu_0 \\ w_1 \\ \vdots \\ w_N  \end{bmatrix} \right)
\end{align}
where $\tilde{f}_i(\bx) = \sum_{m \neq i} c_m f_m(\bx_n)$.
Next we consider $\xi$:
\begin{align}
  \mathcal{J}[\xi] &\propto 
	\left(\frac{1}{1+e^\xi}\right)^{\mu_0'} \left(\frac{1}{1+e^{-\xi}}\right)^{\mu_0'}
	 \times
	 \prod_{n=1}^N 
	   \left(\frac{1}{1 + \exp\left[-y_n \xi \right]}  \right)^{1-w_{n}}
	 \\  &\equiv
	 \text{v-Log}\left(
	   \begin{bmatrix} +1 \\ -1 \end{bmatrix} , 
		 \begin{bmatrix}  0 \\  0  \end{bmatrix} , 
		 \begin{bmatrix} \mu_0' + \sum_{y_n = -1} (1-w_{n}) \\ \mu_0' + \sum_{y_n = +1} (1-w_{n})    \end{bmatrix} \right)
\end{align}
For $w_i$, we have
\begin{align}
  \mathcal{J}[w_i] &\propto 
	  \theta^{w_{i}} (1-\theta)^{1-w_{i}}
	  \notag \\ &\quad
	 \times
	   \left(\frac{1}{1 + \exp\left[-y_i \sum_{m=1}^M c_m f_m(\bx_i)\right]}  \right)^{w_{i}}
	   \left(\frac{1}{1 + \exp\left[-y_i \xi \right]}  \right)^{1-w_{i}}
	 \\  &\equiv
	 \text{Bernoulli}\left(
	 \frac{ \frac{\theta}{1 + \exp\left[-y_i \sum_{m=1}^M c_m f_m(\bx_i)\right]} }
	  {\frac{\theta}{1 + \exp\left[-y_i \sum_{m=1}^M c_m f_m(\bx_i)\right]} + 
		\frac{1-\theta}{1 + \exp\left[-y_i \xi \right]} } 
	 \right) \ .
\end{align}
Finally, for $\btheta$ we have
\begin{align}
  \mathcal{J}[\btheta] &\propto 
	 \theta^{\zeta_1 - 1} (1-\theta)^{\zeta_2 - 1} 
	 \times
	 \prod_{n=1}^N \theta^{w_{n}} (1-\theta)^{1-w_{n}}
	 \\  &\equiv
	 \text{Beta} \left( \zeta_1 + \sum_{n=1}^N w_{n} ,\zeta_2 + \sum_{n=1}^N (1-w_{n}) \right)  \ \ .
\end{align}
The Gibbs sampler is given in Algorithm~\ref{alg:gibbs}.

\begin{algorithm}[th]
   \caption{Gibbs Sampler}
   \label{alg:gibbs}
\begin{algorithmic}
   \STATE {\bfseries Input:} $\{(\bx_n,y_n)\}_{n=1}^N$, $\mathcal{F}$,$\mu_0\in\REAL_+$,$\mu_0'\in\REAL_+$,$\bzeta\in\REAL_+^3$
   \STATE Initialize $\bEta \in \REAL_+^2$, $\bomega \in \REAL_+^2$, and $\bphi \in [0,1]^N$
   \FOR{$t=1$ {\bfseries to} $T$}
   \vskip 0.5ex
	 \FOR{$i=1$ {\bfseries to} $M$}
	 \vskip 0.5ex
	 \STATE $c_i \sim \text{v-Log}\left(
	   \begin{bmatrix} +1 \\ -1 \\ -y_1 f_i(\bx_1) \\ \vdots \\ -y_n f_i(\bx_N)  \end{bmatrix} , 
		 \begin{bmatrix}  0 \\  0 \\ -\tilde{f}_i(\bx_1) f_i(\bx_1) \\ \vdots \\ -\tilde{f}_i(\bx_N) f_i(\bx_N)  \end{bmatrix} , 
		 \begin{bmatrix}  \mu_0 \\  \mu_0 \\ w_1 \\ \vdots \\ w_N  \end{bmatrix} \right)$
	 \vskip 0.5ex
	 \ENDFOR
	  \STATE $\xi \sim 
	   \text{v-Log}\left(
	   \begin{bmatrix} +1 \\ -1 \end{bmatrix} , 
		 \begin{bmatrix}  0 \\  0  \end{bmatrix} , 
		 \begin{bmatrix} \mu_0' + \sum_{y_n = -1} (1-w_{n}) \\ \mu_0' + \sum_{y_n = +1} (1-w_{n})    \end{bmatrix} \right)$
	 \FOR{$i=1$ {\bfseries to} $N$}
	 \vskip 0.5ex
	  \STATE $w_i \sim
	 \text{Bernoulli}\left(
	 \frac{ \frac{\theta}{1 + \exp\left[-y_i \sum_{m=1}^M c_m f_m(\bx_i)\right]} }
	  {\frac{\theta}{1 + \exp\left[-y_i \sum_{m=1}^M c_m f_m(\bx_i)\right]} + 
		\frac{1-\theta}{1 + \exp\left[-y_i \xi \right]} } 
	 \right) $
	 \vskip 0.5ex
	 \ENDFOR
	 \vskip 0.5ex
	  \STATE $\theta \sim
	     \text{Beta} \left( \zeta_1 + \sum_{n=1}^N w_{n} ,\zeta_2 + \sum_{n=1}^N (1-w_{n}) \right)$
	  \ENDFOR
   \STATE {\bfseries Output:} Samples from the posterior
\end{algorithmic}
\end{algorithm}

\newpage
\section{The Variational Updates}   \label{supp:sec:vi}
\subsection{The weight update ($c$)} \label{supp:sec:cupdate}

Isolating the terms of the log-joint ($\mathcal{L}$) involving $c$, we obtain
\begin{align}
  \mathcal{L}_1 &= \const  -\mu_0\log(1+e^c) - \mu_0\log(1+e^{-c})
     \notag \\ &\qquad
	 - \sum_{n=1}^N w_{n} \log(1 + \exp[-y_n(H(\bx_n){+}ch(\bx_n))]) 
\end{align}
We only require the expectation with respect to $w_{1:N}$:
\begin{align}
  \log q^*(c \mid \bbeta, \bgamma, \bmu)  &=  \const -\mu_0\log(1+e^c) - \mu_0\log(1+e^{-c}) 
    \notag \\ &\qquad
	 \label{supp:eq:updatec}
	 -\sum_{n=1}^N \phi_{n} \log(1 + \exp[-y_n(H(\bx_n){+}ch(\bx_n))])
\end{align}
Here we note that $\mp y_n\{H(\bx_n) {+} ch(\bx_n)\} = \mp y_n h(\bx_n) \{ c {+} H(\bx_n) h(\bx_n) \}$ from the fact that $h(\bx_n) \in \{-1,+1\}$.
To the exclusion of $\phi_n$, this manipulation presumes binary logistics, $\text{b-Log}(c,\mp y_n h(\bx_n),-H(\bx_n) h(\bx_n))$, and so conjugacy will come into play.
The form presented in \eqref{supp:eq:updatec} parametrizes a versatile logistic distribution with parameters of length $N+2$ given by
\begin{align}
  \bbeta(h) &\triangleq 
	\begin{bmatrix} +1 \\ -1 \\ -y_1 h(\bx_1) \\ \vdots \\ -y_N h(\bx_N) \end{bmatrix}
	   &
  \bgamma(H,h) &\triangleq 
	\begin{bmatrix} 0 \\ 0 \\ -H(\bx_1) h(\bx_1) \\ \vdots \\ -H(\bx_N) h(\bx_N) \end{bmatrix}
       &
  \bmu(\bphi) &\triangleq
	\begin{bmatrix} \mu_0 \\ \mu_0 \\ \phi_1 \\ \vdots \\ \phi_N \end{bmatrix}
    \label{supp:eq:bgm} \ \ \  .
\end{align}

Before proceeding to the next update, we address modal estimation of the versatile logistic.
Finding the mode requires minimizing the negative log of the density or ${\sum_{k=1}^K\,\mu_k \log(1+e^{\beta_k(z-\gamma_k)})}$ \ 
(the extraneous normalization constant is discarded).
The objective of interest is a weighted LogLoss \cite{collins2002} and minimizing it can be accomplished iteratively.
Alternatively, we can reason that for a fixed $k$, the quantity $\log(1+e^{\beta_k(z-\gamma_k)})$ contributes most to the mode wherever the exponential term is small.
Using a semi-tail approximation, we have $\log(1+e^{\beta_k(z-\gamma_k)}) \approx e^{\tau \beta_k(z-\gamma_k)}$ for positive scalar $\tau$.
Setting $\tau = 1$ best approximates the extreme part of the tail, whereas $\tau = 1/2$ will match the first derivative at $z = \gamma_k$.
If we restrict ourselves to slopes of equal magnitude, i.e., $\vert \beta_k \vert = \beta > 0$, we now minimize
\begin{align}
  e^{\tau\beta z} \sum_{\beta_k>0} \mu_k e^{-\tau\beta\gamma_k} 
  \ +\ e^{-\tau\beta z} \sum_{\beta_k<0} \mu_k e^{\tau\beta\gamma_k} . 
\end{align}
Taking the derivative with respect to $z$ and setting equal to zero yields the approximate mode
\begin{align}
  \alpha = 
  \frac{1}{2\tau\beta} \log \left( \frac{ 
           \sum_{k:\beta_k<0} \mu_k e^{+\tau\beta\gamma_k}
     }{ 
           \sum_{k:\beta_k>0} \mu_k e^{-\tau\beta\gamma_k}
	    } \right) \ .
		\label{supp:eq:appxmode}
\end{align}

\subsection{The noise grade update ($\xi$)}

Starting from the log-likelihood, we have
\begin{align}
  \mathcal{L}_2 &= \const - \mu_0'\log(1+e^\xi) - \mu_0'\log(1+e^{-\xi}) 
  -\sum_{n=1}^N (1-w_{n}) \log(1 + \exp[-y_n\xi]) \ .
\end{align}
We only require the expectation with respect to $w_{1:N}$:
\begin{align}
  \!\!\!\!
  \log q^*(\xi \mid \bomega) &= \const
  -\mu_0'\log(1{+}e^\xi) - \mu_0'\log(1{+}e^{-\xi})
  -\sum_{n=1}^N (1{-}\phi_{n}) \log(1{+}\exp[-y_n\xi])  .
\end{align}
The above parametrizes a versatile logistic with slope vector $\bar{\bu}$, knot vector $\vzero$, and respective multiplicities
\begin{align}
  \omega_1 &= \mu_0' + \sum_{n=1}^N (1-\phi_{n}) \indicator\{y_n = -1\} \\ 
  \omega_2 &= \mu_0' + \sum_{n=1}^N (1-\phi_{n}) \indicator\{y_n = +1\} \ . 
\end{align}

\subsection{The type update ($w_n$)}

The relevant terms of the log-likelihood are
\begin{align}
  \mathcal{L}_3 &= \const + w_n \log \theta + (1-w_n) \log (1-\theta)
   \notag \\ &\qquad 
   - w_{n} \log(1 + \exp[-y_n(H(\bx_n){+}ch(\bx_n))])
  - (1-w_{n}) \log(1 + \exp[-y_n\xi]) \label{supp:eq:lwn} .
\end{align}
Expectation with respect to $c$ is done via the approximate mode $\alpha$, i.e., we replace $c$ with $\alpha$.
We now consider $\text{v-Log}(\bar{\bu} , \vzero , [\omega_1,\omega_2])$; the approximate posterior of $\xi$.
Elementary calculus reveals the mode is ${\log(\omega_2/\omega_1)}$ which is in exact agreement with \eqref{supp:eq:appxmode} for $\tau = 1/2$.
Suppose $V \sim \text{Beta}(\omega_1,\omega_2)$.
If $Z = {\log( \frac{1}{V} - 1)}$, then it can be shown that $Z \sim \text{v-Log}(\bar{\bu} , \vzero , [\omega_1,\omega_2])$ (see \S\ref{supp:sec:vlog11}).
In this particular case, both distributions have the same normalization constant of $\frac{\Upgamma(\omega_1)\Upgamma(\omega_2)}{\Upgamma(\omega_1{+}\omega_2)}$ and leveraging the relationship of $V$ with $X$, we evaluate
\begin{align}
  \EX_z\{ \log(1+e^{+Z}) \} &= -\EX_v\{ \log V \} = \uppsi(\omega_0) - \uppsi(\omega_1) \label{supp:eq:vlog21}  \\
  \EX_z\{ \log(1+e^{-Z}) \} &= -\EX_v\{ \log(1{-}V) \} = \uppsi(\omega_0) - \uppsi(\omega_2) \label{supp:eq:vlog22}  \ ,
\end{align}
where $\omega_0 \triangleq \omega_1 + \omega_2$ and $\uppsi(\cdot)$ is the digamma function.
(The above two expectations can be used to derive the differential entropy.)

Taking the expectation of \eqref{supp:eq:lwn} with respect to $\xi$ using \eqref{supp:eq:vlog21} and \eqref{supp:eq:vlog22}, we obtain
\begin{align}
  \log q^*(w_n \mid \phi_n) &\approx  \const 
   + w_{n} [\uppsi(\eta_1) - \uppsi(\eta_0)] 
   + (1-w_{n}) [\uppsi(\eta_2) - \uppsi(\eta_0)] \notag 
	 \\ &\ -w_{n} \log(1 + \exp[-y_n(H(\bx_n){+}\alpha h(\bx_n))])   \notag 
     \\ &\ -(1-w_{n}) [\uppsi(\omega_0) - \uppsi(\omega_2)]\indicator\{y_n=+1\} \notag
     \\ &\ -(1-w_{n}) [\uppsi(\omega_0) - \uppsi(\omega_1)]\indicator\{y_n=-1\}  \label{supp:eq:qwn}\ ,
\end{align}
where $\eta_0 \triangleq \eta_1 + \eta_2 $.
If 
\begin{align}
  \!\!\!
  \kappa_{n}(H,h,\alpha) &\triangleq \frac{ \exp\left[ \uppsi(\eta_1) - \uppsi(\eta_2) + \uppsi(\omega_0) - \uppsi(\omega_2)\indicator{\{y_n=+1\}} - \uppsi(\omega_1)\indicator{\{y_n=-1\}} \right] }{ 1 + \exp[-y_n(H(\bx_n) + \alpha h(\bx_n))]} 
   \ , \label{supp:eq:kappa}
\end{align}
then \eqref{supp:eq:qwn} describes a Bernoulli distribution for $w_n$ with parameter $\phi_{n} = \kappa_{n}/(1+\kappa_{n})$.

\subsection{The type prior update ($\theta$)}

Again, starting from the log-likelihood, we have
\begin{align}
   \mathcal{L}_4 = \const +
	 (\zeta_1 - 1) \log \theta + (\zeta_2 - 1) \log (1-\theta) + 
	 \sum_{n=1}^N \left[ w_n \log \theta + (1-w_n) \log (1-\theta) \right]
\end{align}
Taking the expectation with respect to $w_{1:N}$ gives
\begin{align}
  \log q^*(\theta \mid \bEta) &= \const +
  \left( \zeta_1 - 1 + \sum_{n=1}^N \phi_n \right) \log \theta +
  \left( \zeta_2 - 1 + \sum_{n=1}^N (1-\phi_n) \right) \log (1-\theta)
\end{align}
which corresponds to a Beta distribution with parameters
\begin{align}
  \eta_1 &= \zeta_1 + \sum_{n=1}^N \phi_n  \\
  \eta_2 &= \zeta_2 + \sum_{n=1}^N (1-\phi_n) \ \ \ .
\end{align}


\newpage

\newpage

\section{The [Dynamic] ELBO}   \label{supp:sec:elbo}

We are interested in the additive change of the ELBO.
The log of the joint is
\begin{align}
  \mathcal{L} &= \const  
   - \mu_0'\log(1+e^\xi) - \mu_0'\log(1+e^{-\xi})   -\mu_0 \log(1+e^c) - \mu_0 \log(1+e^{-c}) 
        \notag \\ &\qquad
   + (\zeta_1-1) \log \theta + (\zeta_2-1) \log(1-\theta) 
   + \sum_{n=1}^N w_n \log \theta + (1-w_n) \log (1-\theta)
        \notag \\ &\qquad
   - \sum_{n=1}^N w_n \log(1 + \exp[-y_n (H(\bx_n)+ch(\bx_n))]) + (1-w_n) \log(1 + \exp[-y_n\xi]) 
  \label{supp:eq:smjointupdatelog1}
\end{align}
We now take the expectations with respect to the auxiliary distributions:
\begin{align}
  \mathcal{\EX_Q\{L\}}
  &= \const - \sum_{j=1}^2 \mu_0'(\uppsi(\omega_1+\omega_2)-\uppsi(\omega_j)) 
					 -\mu_0 \EX_c\{\log(1+e^c)\} - \mu_0 \EX_c\{\log(1+e^{-c})\} \notag \\
					 &\ - \sum_{j=1}^2 (\zeta_{j}-1) (\uppsi(\eta_1+\eta_2) - \uppsi(\eta_j) )
					 \notag \\ &\ 
  -\sum_{n=1}^N \phi_{n} (\uppsi(\eta_1+\eta_2) - \uppsi(\eta_1) ) + (1-\phi_{n}) (\uppsi(\eta_1+\eta_2) - \uppsi(\eta_2) ) \notag 
  \\
  &\ -\sum_{n=1}^N (1-\phi_{n}) \left[ (\uppsi(\omega_1+\omega_2) - \uppsi(\omega_1))\indicator\{y_n = -1\} + (\uppsi(\omega_1+\omega_2) - \uppsi(\omega_2))\indicator\{y_n = +1\} \right] \notag \\
	&\ -\sum_{n=1}^N
	\phi_{n} \EX_c\{\log(1 + \exp[-y_n(H(\bx_n) + ch(\bx_n))])\}  \\
  &= \const -  \sum_{j=1}^2 \mu_0'(\uppsi(\omega_1+\omega_2)-\uppsi(\omega_j))
					 - \sum_{j=1}^2 (\zeta_{j}-1) (\uppsi(\eta_1+\eta_2) - \uppsi(\eta_j) )
					 \notag \\ &\ 
  -\sum_{n=1}^N \phi_{n} (\uppsi(\eta_1+\eta_2) - \uppsi(\eta_1) ) + (1-\phi_{n}) (\uppsi(\eta_1+\eta_2) - \uppsi(\eta_2) ) \notag 
  \\
  &\ -\sum_{n=1}^N (1-\phi_{n}) \left[ (\uppsi(\omega_1+\omega_2) - \uppsi(\omega_1))\indicator\{y_n = -1\} + (\uppsi(\omega_1+\omega_2) - \uppsi(\omega_2))\indicator\{y_n = +1\} \right] \notag \\
  &\ -\sum_{k=1}^{N+2} \mu_k \EX_c\{\log(1+e^{\beta_k(c-\gamma_k)}\}
  \label{supp:eq:smjointupdatelog3} 
\end{align}
Expectation with respect to $c$ refers to v-Log$(\bbeta,\bgamma,\bmu)$.
Let $B_c$ denote the normalization constant for this density.
We now look at the entropy of the auxiliary distributions.
\begin{itemize}
  \item $c$:
	\begin{align}
         \log B_c + \sum_{k=1}^{N+2} \mu_k \EX_c\{\log(1+e^{\beta_k(c-\gamma_k)}\}
	\end{align}
  \item $\xi$: From \S\ref{supp:sec:vlog11}, we have
	\begin{align}
	  \log\left( \frac{ \Upgamma(\omega_1)\Upgamma(\omega_2) }{\Upgamma(\omega_1+\omega_2)} \right) +
	  (\omega_1+\omega_2) \uppsi(\omega_1+\omega_2) - \sum_{j=1}^2 \omega_j \uppsi(\omega_j)
	\end{align}
  \item $w_{1:N}$:
	\begin{align}
	  -\sum_{n=1}^N \phi_{n} \log \phi_{n} + (1-\phi_n) \log(1-\phi_n)
	\end{align}
  \item $\theta$:
	  \begin{align}
		\log\left( \frac{\Upgamma(\eta_1)\Upgamma(\eta_2)}{\Upgamma(\eta_1+\eta_2)} \right) + (\eta_1+\eta_2-2)\uppsi(\eta_1+\eta_2) -\sum_{j=1}^2 (\eta_j-1)\uppsi(\eta_j)
	  \end{align}
\end{itemize}
Combining into a single expression, we obtain
\begin{align}
  \textsc{ELBO} &= \const + \log B_c + 
      \sum_{j=1}^2 \log\Upgamma(\omega_j) - \log\Upgamma(\omega_0)
	  + (\omega_0 - 2\mu_0')\uppsi(\omega_0) - \sum_{j=1}^2 (\omega_j-\mu_0')\uppsi(\omega_j) 
	  \notag \\ &\qquad
      + \sum_{j=1}^2 \log\Upgamma(\eta_j) - \log\Upgamma(\eta_0)
	  + (\eta_0 - \zeta_0) \uppsi(\eta_0) - \sum_{j=1}^2 (\eta_j-\zeta_j)\uppsi(\eta_j) 
	  \notag \\ &\qquad
	  -\sum_{n=1}^N \left[ \phi_n \log \phi_n + (1-\phi_n) \log (1-\phi_n) \right]
	  -N \uppsi(\eta_0) + \sum_{n=1}^N \left[ \uppsi(\eta_1) \phi_{n} + \uppsi(\eta_2) (1 - \phi_n) \right]
	  \notag \\ &\qquad
	  - \uppsi(\omega_0) \sum_{n=1}^N (1-\phi_{n})
	  + \uppsi(\omega_1) \sum_{n:y_n=-1} (1-\phi_{n})
	  + \uppsi(\omega_2) \sum_{n:y_n=+1} (1-\phi_{n}) \ \ ,
  \label{supp:eq:elbo}
\end{align}
where the $0$-subscript denotes the vector sum (except for $\mu_0'$).

\subsection{Computing the normalization constant} \label{supp:sec:nc}

The change of the ELBO requires the computation of $\log B_c$.
Referring back to \eqref{supp:eq:vlogx1} we are interested in $B = \int_{-\infty}^{+\infty} f(z) dz$. 
To find $B$ we will employ numerical integration -- a reasonable approach for a function of a single variable. 
Now, consider 
\begin{align}
  g(z) &= -\log f(z) = \sum_{k=1}^K \mu_k \log\left( 1 + e^{\beta_k(z-\gamma)} \right) \ \ ,
\end{align}
and so $B = \int_{-\infty}^{+\infty} e^{-g(z)} dz$. 
The function $g$ is positive and convex.

One problem with using numerical integration blindly is that for large $z$, $\log(1+e^z)$ might return infinity.
For example, in evaluating $\log(1+e^{5000})$, computational software will first perform $e^{5000}$ and return a value of infinity.
Subsequently, adding one and taking the log will also return infinity.
This motivates the following Lemma.

\begin{lem}
  For $\POSP{z} \triangleq z \indicator\{z > 0\}$ we have $\log(1+e^z) = \log(1+e^{-\vert z \vert}) + \POSP{z}$.
\end{lem}
\begin{proof}
  If $z \leq 0$ then $z = -\vert z \vert$ and $\POSP{z} = 0$ yielding equality.
  If $z > 0$ then $z = \vert z \vert$ and $\POSP{z} = z$.  We have $\log(1+e^{-z}) + z = \log(1+e^{-z}) + \log e^z = \log(1 + e^z)$. 
\end{proof}

Utilizing the above Lemma to evaluate $\log(1+e^z)$ ensures that infinite values are not returned from software.
Revisiting the previous example, $\log(1+e^{5000}) \rightarrow \log(1+e^{-5000}) + 5000 \approx 0 + 5000 = 5000$.

Now suppose that $1000$ is a lower bound on $g(z)$.
A numerical integration procedure would have to deal with numbers on the order of $e^{-1000}$, leading to a zero estimate of $B$.
To avoid this pitfall we translate $g(z)$.
Utilizing a Golden Section Search can produce the minimum value of $g(z)$.
Let $\bar{z}$ be the scalar such that $g'(\bar{z}) = 0$, i.e., $\bar{z}$ is the global minimizer of $g$.
We can now consider
\begin{align}
  B &= \int_{-\infty}^{+\infty} \exp[-g(z)] dz
     =  \int_{-\infty}^{+\infty} \exp[-g(z+\bar{z})] dz
	 \\ &= \exp[-g(\bar{z})] \int_{-\infty}^{+\infty} \exp[-(g(z+\bar{z}) - g(\bar{z})) ] dz \\ 
 \log B &= -g(\bar{z}) + \log\left( \int_{-\infty}^{+\infty} \exp[-(g(z+\bar{z}) - g(\bar{z})) ] dz \right) \ .
\end{align}
The translated function $\tilde{g}(z) =  g(z+\bar{z}) - g(\bar{z})$ is nonnegative with $0$ as the global minimizer and $\tilde{g}(0) = 0$. 

Our last step before using numerical integration is the contraction of the integration limits.
Using the substitution $u = \tan^{-1}x$ (or $x = \tan u$) we arrive at
\begin{align}
  B &=
  \exp[-g(\bar{z})] \int_{-\pi/2}^{+\pi/2}  \exp[-\tilde{g}(\tan u)] \sec^2(u) du \\ 
  \log B &=
  -g(\bar{z}) + \log\left(\int_{-\pi/2}^{+\pi/2} \exp[-\tilde{g}(\tan u)] \sec^2(u) du \right) \ .
\end{align}
This final integral serves as the input to a numerical integrator to produce $\log B$.

\newpage
\section{Approximate mode simplification }  \label{supp:sec:modeappx} 

Recall: $y_n^2 = 1$, $y_n h(\bx_n) = +1 \Leftrightarrow y_n = h(\bx_n)$ and $y_n h(\bx_n) = -1 \Leftrightarrow y_n \neq h(\bx_n)$.  We form
\begin{align}
  Z &= \sum_{n=1}^N \phi_{n} e^{-\tau y_n H(\bx_n)} & d_n &= \frac{\phi_{n} e^{-\tau y_n H(\bx_n)}}{Z} &
  \varepsilon &= \sum_{n=1}^N d_n \indicator\{y_n \neq h(\bx_n)\} \ .
\end{align}
We have
\begin{align}
  &\quad\frac{1}{2\tau} \log\left( \frac{ \mu_0
      + \sum_{n=1}^N \phi_{n} e^{-\tau H(\bx_n) h(\bx_n)} \indicator\{y_n = h(\bx_n)\} 
	        }{ \mu_0
      + \sum_{n=1}^N \phi_{n} e^{+\tau H(\bx_n) h(\bx_n)} \indicator\{y_n \neq h(\bx_n)\} 
	        }\right)
  \\ &=
       \frac{1}{2\tau} \log\left( \frac{ \mu_0
      + \sum_{n=1}^N \phi_{n} e^{-\tau y_n H(\bx_n) y_n h(\bx_n)} \indicator\{y_n = h(\bx_n)\} 
	        }{ \mu_0
      + \sum_{n=1}^N \phi_{n} e^{+\tau y_n H(\bx_n) y_n h(\bx_n)} \indicator\{y_n \neq h(\bx_n)\} 
	        }\right)
  \\ &=
       \frac{1}{2\tau} \log\left( \frac{ \mu_0
      + \sum_{n=1}^N \phi_{n} e^{-\tau y_n H(\bx_n)} \indicator\{y_n = h(\bx_n)\} 
	        }{ \mu_0
      + \sum_{n=1}^N \phi_{n} e^{-\tau y_n H(\bx_n)} \indicator\{y_n \neq h(\bx_n)\} 
	        }\right)
  \\ &=
       \frac{1}{2\tau} \log\left( \frac{ \mu_0/Z
      + \sum_{n=1}^N \phi_{n} e^{-\tau y_n H(\bx_n)} \indicator\{y_n = h(\bx_n)\}/Z 
	        }{ \mu_0/Z
      + \sum_{n=1}^N \phi_{n} e^{-\tau y_n H(\bx_n)} \indicator\{y_n \neq h(\bx_n)\}/Z
	        }\right)
  \\ &=
       \frac{1}{2\tau} \log\left( \frac{ \mu_0/Z
      + \sum_{n=1}^N d_n \indicator\{y_n = h(\bx_n)\} 
	        }{ \mu_0/Z
      + \sum_{n=1}^N d_n \indicator\{y_n \neq h(\bx_n)\}
	        }\right)
  \\ &=
       \frac{1}{2\tau} \log\left( \frac{ \mu_0/Z
      + \sum_{n=1}^N d_n (1 - \indicator\{y_n \neq h(\bx_n)\}) 
	        }{ \mu_0/Z
      + \sum_{n=1}^N d_n \indicator\{y_n \neq h(\bx_n)\}
	        }\right)
  \\ &=
       \frac{1}{2\tau} \log\left( \frac{ \mu_0/Z + 1
      - \sum_{n=1}^N d_n \indicator\{y_n \neq h(\bx_n)\}
	        }{ \mu_0/Z
      + \sum_{n=1}^N d_n \indicator\{y_n \neq h(\bx_n)\}
	        }\right)
  \\ &=
       \frac{1}{2\tau} \log\left( \frac{ \mu_0/Z + 1 - \varepsilon
	        }{ \mu_0/Z
      + \varepsilon
	        }\right)
\end{align}

\newpage
\section{Matlab code for Long-Servedio Data} \label{supp:sec:longserv}

The code below can be used to generate samples.  In the paper, we called the function with \texttt{n=10} and \texttt{eta=0.20}.

\begin{lstlisting}
function [X y] = longservedio(m,n,eta)
% [X y] = longservedio(m,n,eta)
% inputs:
%   m   ->   # of samples
%   n   ->   data will have 2n+11 dimensions
%   eta ->   Bayes error
% outputs:
%   X   ->   matrix m.by.(2n+11)
%   y   ->   m-vector of +1/-1 labels 
%
% Source: 
%    Boosting: Foundations and Algorithms
%    Shapire & Freund, MIT Press 2012
%    [12.3]
%
% m-file author:  Alex Lorbert
%

k = 2*n+1;
X = zeros(m,2*n+11);
y = zeros(m,1);

for i=1:m
    t = rand(1);
    if t < 1/4       % with probability 0.25
        x1 = ones(1,2*n+1);
        x2 = get_random_bv(10,0);
    elseif t < 3/4   % with probability 0.5
        x1 = get_random_bv(k,1);
        x2 = get_random_bv(10,-2);
    else             % with probability 0.25
        x1 = get_random_bv(k,1);
        x2 = ones(1,10);
    end
    X(i,:) = [x1 x2];
    s = rand(1);
    if s < 1-eta
        y(i) = 1;
    else
        y(i) = -1;
    end
end

return


function x = get_random_bv(n,k)   % get random binary vector of dimension n
                                  % such that the sum equals k
    m = (k + n)/2;       % n and k will have same parity by construction
    x = ones(1,n);
    x(1:m) = -1;    
    x = x(randperm(n));  % randomly permute 
    return

\end{lstlisting}



\end{document}